\newif\iftwocol\twocolfalse
\documentclass[letterpaper,final]{article}
\newif\ifblind\blindfalse

\PassOptionsToPackage{xcolor}{cmyk}

\usepackage{tex/pkgs/aaai21_full}
\usepackage[table]{xcolor}
\usepackage{ifdraft}

\ifoptionfinal{}{
\paperwidth=\dimexpr \paperwidth + 6cm\relax
\oddsidemargin=\dimexpr\oddsidemargin + 3cm\relax
\evensidemargin=\dimexpr\evensidemargin + 3cm\relax
\marginparwidth=\dimexpr \marginparwidth + 2cm\relax
}

\title{Agent Incentives: A Causal Perspective}

\author {
Tom Everitt,\textsuperscript{*\rm 1}
    Ryan Carey,\textsuperscript{*\rm 2}
    Eric D.\ Langlois,\textsuperscript{*\rm 1,3,4}
    Pedro A.\ Ortega,\textsuperscript{\rm 1}
    Shane Legg\textsuperscript{\rm 1}
    \\
}
\affiliations {
\textsuperscript{\rm 1}DeepMind,
    \textsuperscript{\rm 2}University of Oxford,
    \textsuperscript{\rm 3}University of Toronto,
    \textsuperscript{\rm 4}Vector Institute,
    \textsuperscript{*}Equal Contribution \\
    tomeveritt@google.com, ry.duff@gmail.com, edl@cs.toronto.edu, pedroortega@google.com
}

\usepackage{tex/pkgs/setup}
\usepackage{tex/pkgs/symbols}

\usepackage[decisionutilitycolor]{tex/pkgs/influence-diagrams}

\usepackage{amsmath}    \usepackage{amssymb}    \usepackage{amsthm}     \usepackage{thmtools}
\usepackage{thm-restate}
\usepackage{bm}         \usepackage{booktabs}   \usepackage{enumitem}   \usepackage[utf8]{inputenc}  \usepackage{mathtools}  \usepackage{microtype}  

\usepackage{natbib} \usepackage{tcolorbox}  \usepackage{xcolor}     \usepackage{subcaption}

\usepackage{cleveref}
\usepackage{nameref}
\newcommand{\posscite}[1]{\citeauthor{#1}'s (\citeyear{#1})}

\date{}

\begin{document}
\maketitle

\begin{abstract}

We present a framework for analysing agent incentives using causal influence
diagrams. We establish that a well-known criterion for value of information is
complete.
We propose a new graphical criterion for value of control, establishing its
soundness and completeness.
We also introduce two new concepts for incentive analysis: response incentives
indicate which changes in the environment affect an optimal decision, while
instrumental control incentives establish whether an agent can influence its utility
via a variable X.
For both new concepts, we provide sound and complete graphical criteria.
We show by
example how these results can help with evaluating the safety and fairness of an
AI system.

\end{abstract}

\section{Introduction}\label{sec:introduction}
A recurring question in AI research is how to choose an objective to induce safe
and fair
behaviour \citep{oneil2016weapons,russell2019human}.
In a given setup, will an optimal policy depend on a sensitive attribute, or seek to influence an important variable?
For example, consider the following two incentive design problems, to which we will return throughout the paper:~\looseness=-1

\begin{example}[Grade prediction]
To decide which applicants to admit, a university uses a model to predict the grades of new students.
The university would like the system to predict accurately, without treating students differently based on their gender or race
(see \cref{fig:race-preview}).~\looseness=-1 \end{example}

\begin{example}[Content recommendation]
An AI algorithm has the task of recommending a series of posts to a user.
The designers want the algorithm to present content adapted to each user's interests to optimize clicks.
However, they do not want the algorithm to use polarising content to manipulate the user into clicking more predictably
(\cref{fig:fci-preview}).~\looseness=-1
\end{example}

\paragraph{Contributions}
This paper provides a common language for incentive analysis, based on
influence diagrams \citep{howard1990influence} and causal models \citep{Pearl2009}.
Traditionally, influence diagrams have been used to help decision-makers make
better decisions.
Here, we invert the perspective, and use the diagrams to understand and predict
the behaviour of machine learning systems trained to optimize an objective in a
given environment.
To facilitate this analysis, we prove a number of relevant theorems and
introduce two new concepts:
\begin{itemize}
\item \emph{Value of Information}  (VoI):
  First defined by \citet{Howard1966}, a graphical criterion for detecting positive
  VoI in influence diagrams were proposed and proven sound by
  \citet{fagiuoli1998note}, \citet{Lauritzen2001}, and \citet{Shachter2016}.
  Here we offer the first correct completeness proof, showing that the graphical
  criterion is unique and cannot be further improved upon.
\item
  \emph{Value of Control} (VoC):
  Defined by \citet{shachter1986evaluating}, \citet{Matheson1990}, and \citet{Shachter2010}, an
  incomplete graphical criterion was discussed by \citet{shachter1986evaluating}.
  Here we provide a complete graphical criterion, along with both soundness and
  completeness proofs.
\item
  \emph{Instrumental Control incentive} (ICI):
  We propose a refinement of VoC to nodes the agent can influence
  with its decision.
  Conceptually, this is a hybrid of VoC and \emph{responsiveness} \citep{Shachter2016}.
We offer a formal definition of instrumental control incentives
  based on nested counterfactuals, and
  establish a sound and complete graphical criterion.
\item
  \emph{Response incentive} (RI):
  Which changes in the environment does an optimal policy respond to?
  This is a central problem in fairness
  and AI safety (e.g.\ \citealp{kusner2017counterfactual,Hadfield-Menell2016osg}).
  Again, we give a formal definition, and a
  sound and complete graphical criterion.
\end{itemize}

Our analysis focuses on influence diagrams with a single-decision.
This single-decision setting is adequate to model supervised learning, (contextual)
bandits, and the choice of a policy in an MDP. Previous work has also discussed ways
to transform a multi-decision setting into a single-decision setting by imputing policies
to later decisions \citep{Shachter2016}.

\begin{figure*}
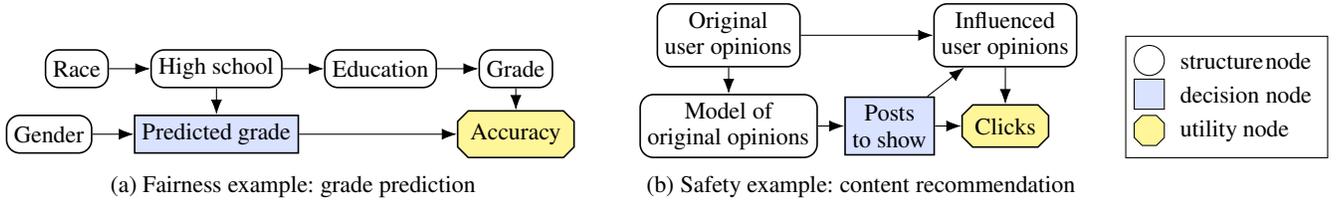

\begin{subfigure}[t]{0.45\textwidth}
    \centering
      \begin{influence-diagram}
    \setrectangularnodes
    \setcompactsize
    \tikzset{node distance=3.5mm and 5.5mm}

\node (R) [] {Race};
      \node (S) [right= of R] {High school};
      \node (E) [right= of S] {Education};
      \node (Gr) [right=of E] {Grade};
\node (D) [below=of S,decision] {Predicted grade};
      \node (Ge) [left= of D] {Gender};

      \node (U) at (Gr|-D) [utility] {Accuracy};

    \draw[->]
      (R) edge (S)
      (S) edge (E)
      (S) edge (D)
      (E) edge (Gr)
      (Gr) edge (U)
      (Ge) edge (D)
      (D) edge (U)
    ;
\end{influence-diagram}
 \caption{Fairness example: grade prediction}\label{fig:race-preview}
  \end{subfigure}\hfill
  \begin{subfigure}[t]{0.34\textwidth}
      \centering
      \begin{influence-diagram}
  \setrectangularnodes
  \setcompactsize
  \tikzset{node distance=4mm and 3.5mm}
  \node (D) [decision, anchor=west] {Posts\\ to show};
  \node (M) [left = of D] {Model of\\ original opinions};
  \node (P1) [above =3.5mm of M] {Original\\ user opinions};
  \node (U) [right = of D, utility] {Clicks};
  \node (P2) at (U|-P1) {Influenced\\ user opinions};
\draw[->]
    (P1) edge (M)
    (M) edge[information] (D)
    (P1) edge (P2)
    (D) edge (P2)
    (D) edge (U)
    (P2) edge (U)
  ;
\end{influence-diagram}
       \caption{Safety example: content recommendation}\label{fig:fci-preview}
  \end{subfigure}\hfill
  \begin{subfigure}[t]{0.16\textwidth}
    \begin{influence-diagram}
    \cidlegend[]{
      \legendrow{}{structure\! node} \\
      \legendrow{decision}{decision node}\\
      \legendrow{utilityc, chamfered rectangle xsep=1.5pt, chamfered rectangle ysep=1.5pt}{utility node}\\
}
\end{influence-diagram}
  \end{subfigure}
  \caption{Two examples of decision problems represented as causal influence diagrams.
      In a) a predictor at a hypothetical university aims to estimate a student's grade, using as inputs
      their gender and the high school they attended.
      We ask whether the predictor is incentivised to behave in a discriminatory manner
      with respect to the students' gender and race.
      In this hypothetical cohort of students, performance is assumed to be a function of the quality of the high-school education they received.
      A student's high-school is assumed to be impacted by their race, and can affect the quality of their education.
      Gender, however, is assumed not to have an effect.
      In b) the goal of a content recommendation system is to choose posts that will maximise the user's click rate.
      However, the system's designers prefer the system not to manipulate the user's opinions in order
      to obtain more clicks.~\looseness=-1
      }
\label{fig:cid-preview}
\end{figure*}

 \paragraph{Applicability}
This paper combines material from two preprints
\citep{everitt2019understanding,carey2020incentives}.
Since the release of these preprints,
the unified language of causal influence diagrams
have already aided in the understanding of incentive problems such as
an agent's redirectability, ambition, tendency to tamper with reward, and other properties
\citep{Armstrong2020pitfalls,Holtman2020,cohen2020unambitious,everitt2019tampering,Everitt2019modeling,langlois2021rl}.

 \section{Setup}\label{sec:setup}
To analyse agents' incentives, we will need a graphical framework with
the causal properties of a structural causal model and the node categories
of an influence diagram.
This section will define such a model
after reviewing structural causal models and influence diagrams.

\subsection{Structural Causal Models}

Structural causal models (SCMs) \citet{Pearl2009}
are a type of causal model where all randomness is consigned to \emph{exogenous}
variables, while deterministic \emph{structural} functions relate the
\emph{endogenous} variables to each other and to the exogenous ones.
As demonstrated by \citet{Pearl2009}, this \emph{structural} approach
has significant benefits over traditional causal Bayesian networks for
analysing (nested) counterfactuals and ``individual-level'' effects.

\begin{definition}[Structural causal model; {\citealp[Chapter 7]{Pearl2009}}]\label{def:scm}
    A \emph{structural causal model} (with independent errors) is a tuple
    $\langle \exovars, \evars, \structfns, P\rangle$, where $\exovars$ is a set of exogenous
    variables; $\evars$ is a set of endogenous variables;
    and $\structfns= \setfor{\fv{\evar}}{\evar \in \evars}$ is a collection of
    functions, one for each $\evar$.
    Each function $\fv{\evar}\colon \dom(\Pav{\evar} \cup \{\exovarv{\evar}\}) \to \dom({\evar})$
    specifies the value of $\evar$ in terms of
    the values of the corresponding exogenous variable $\exovarv{\evar}$
    and endogenous
    parents $\Pav{\evar} \subset \evars$, where
    these functional dependencies are acyclic.
    The domain of a variable $\evar$ is $\dom(\evar)$
    and for a set of variables, ${\dom(\sW) := \bigtimes_{W \in \sW}{\dom(W)}}$.
The uncertainty is encoded through a probability distribution $P(\exovals)$ such that the exogenous variables are mutually independent.~\looseness=-1
\end{definition}

For example, \cref{fig:counterfactual1} shows an SCM that models how
\emph{posts} ($D$) can influence a user's \emph{opinion} ($O$) and \emph{clicks}
($U$).x

The exogenous variables $\exovars$ of an SCM represent factors that are not
modelled.
For any value $\exovars = \exovals$ of the exogenous variables, the value of
any set of variables $\sW \subseteq \evars$ is given by recursive application of
the structural functions $\structfns$ and is denoted by $\sW(\exovals)$.
Together with the distribution $\exoprob(\exovals)$ over exogenous variables, this induces a
joint distribution ${\Pr(\sW = \sw)}
= \sum_{\{\exovals|\sW(\exovals)=\sw\}}{\exoprob(\exovals)}$.~\looseness=-1

SCMs model \emph{causal interventions} that set variables to particular values.
These are defined via submodels:

\begin{definition}[Submodel; {\citealp[Chapter 7]{Pearl2009}}]\label{def:submodel}
Let $\scm = \scmdef$ be an SCM, $\sX$ a set of variables in $\evars$, and
$\sx$ a particular realization of $\sX$.
The submodel $\scm_\sx$ represents the effects of an \emph{intervention}
$\Do(\sX=\sx)$,
and is formally defined as the SCM
${\langle \exovars, \evars, \structfns_\sx, \exoprob \rangle}$
where
    ${\structfns_\sx = \{\fv{\evar} | \evar \notin \sX\} \cup
    {\{\sX = \sx\}}}$.
That is to say, the original functional relationships of $X \in \sX$ are
replaced with the constant functions $X = x$.
\end{definition}

More generally, a \emph{soft intervention} on a variable $X$
in an SCM $\scm$ replaces $\fv{X}$ with a
function $\gx\colon \dom(\Pav{X} \cup \{\exovarv{X}\}) \to \dom(X)$
\citep{eberhardt2007interventions,tian2013causal}.
The probability distribution $\Pr(\sW_{\gx})$
on any $\sW \subseteq \evars$
is defined as the value of $\Pr(\sW)$ in the submodel $\scim_{\gx}$
where $\scim_{\gx}$ is $\scim$
modified by replacing $\fv{X}$ with $\gx$.

If $W$ is a variable in an SCM $\scim$, then $W_{\sx}$ refers to the same variable
in the submodel $\scim_{\sx}$ and is called a \emph{potential response variable}.
In \cref{fig:counterfactual1}, the random variable $O$ represents user opinion
under ``default'' circumstances
while $O_d$ in \cref{fig:counterfactual3} represents the
user's opinion given an intervention $\doo(D=d)$ on the content posted.
Note also how the intervention on $D$ severs the link from $\exovalv{D}$ to $d$
in \cref{fig:counterfactual3}, as the intervention on $D$ overrides the causal
effect from $D$'s parents.
Throughout this paper we use subscripts to indicate submodels or interventions,
and superscripts for indexing.

\begin{figure*}
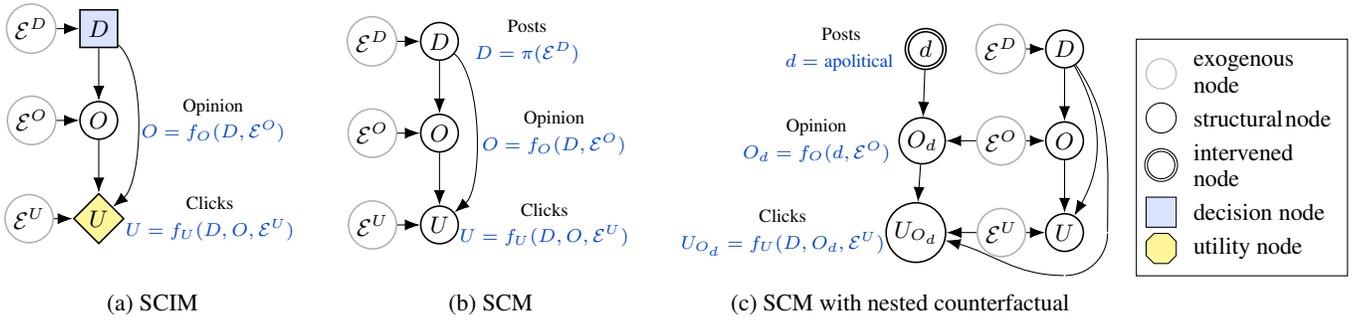

  \centering
\newlength{\mysubfigwidth}
  \pgfmathsetlength{\mysubfigwidth}{40mm}
  \begin{subfigure}[t]{\mysubfigwidth}
    \centering
    \begin{influence-diagram}
      \setcompactsize
      \setinnersep{0.5mm}
      \node (D) [decision] {$D$};
      \node (X) [below =7mm of D] {$O$};
      \node (Y) [below =7mm of X,utility] {$U$};
        \node (ed) [left = 3mm of D, exogenous] {$\exovarv{D}$};
        \node (ex) [left = 3mm of X, exogenous] {$\exovarv{O}$};
        \node (ey) [left = 2.5mm of Y, exogenous] {$\exovarv{U}$};
      \path
      (ed) edge[->] (D)
      (ex) edge[->] (X)
      (ey) edge[->] (Y)
      (D) edge[->] (X)
      (X) edge[->] (Y)
      ;
      \draw[->] (D) to[in=40,out=-40,looseness=.6] (Y);
      \begin{scope}[
        every node/.style={draw=none,rectangle,align=center,inner sep=0mm}
        ]
        \scriptsize
\node[right = 3mm of X] {Opinion \\ \textcolor{dmblue500}{$O = f_O(D, \exovarv{O})$}};
        \node[right = 0mm of Y] {Clicks\\ \textcolor{dmblue500}{$U = f_U(D, O,\exovarv{U})$}};
      \end{scope}
      \node (h2) [below = 4mm of Y, inner sep=0mm, minimum size=0mm] {};
    \end{influence-diagram}
    \caption{SCIM}\label{fig:scim-example}
\end{subfigure}
  \hfill
  \pgfmathsetlength{\mysubfigwidth}{40mm}
  \begin{subfigure}[t]{\mysubfigwidth}
    \centering
    \begin{influence-diagram}
      \setcompactsize
      \setinnersep{0.5mm}
      \node (D) [] {$D$};
      \node (X) [below =7mm of D] {$O$};
      \node (Y) [below =7mm of X] {$U$};
        \node (ed) [left = 3mm of D, exogenous] {$\exovarv{D}$};
        \node (ex) [left = 3mm of X, exogenous] {$\exovarv{O}$};
        \node (ey) [left = 3mm of Y, exogenous] {$\exovarv{U}$};
      \path
      (ed) edge[->] (D)
      (ex) edge[->] (X)
      (ey) edge[->] (Y)
      (D) edge[->] (X)
      (X) edge[->] (Y)
      ;
      \draw[->] (D) to[in=40,out=-40,looseness=.6] (Y);
      \begin{scope}[
        every node/.style={draw=none,rectangle,align=center,inner sep=0mm}
        ]
        \scriptsize
        \node[right = 2mm of D] {Posts \\ \textcolor{dmblue500}{$D = \pi(\exovarv{D})$}};
        \node[right = 3mm of X] {Opinion \\ \textcolor{dmblue500}{$O = f_O(D, \exovarv{O})$}};
        \node[right = 0mm of Y] {Clicks\\ \textcolor{dmblue500}{$U = f_U(D, O,\exovarv{U})$}};
      \end{scope}
      \node (h2) [below = 4mm of Y, inner sep=0mm, minimum size=0mm] {};
    \end{influence-diagram}
    \caption{SCM}\label{fig:counterfactual1}
\end{subfigure}
  \hfill
  \pgfmathsetlength{\mysubfigwidth}{59mm}
  \begin{subfigure}[t]{\mysubfigwidth}
    \centering
    \begin{influence-diagram}
      \setcompactsize
      \setinnersep{0.5mm}
      \node (D) [] {$D$};
      \node (X) [below =7mm of D] {$O$};
      \node (Y) [below =7mm of X] {$U$};
      \node (ed) [left =2.5mm of D, exogenous] {$\exovarv{D}$};
      \node (ex) [left =2.5mm of X, exogenous] {$\exovarv{O}$};
      \node (ey) [left =2.5mm of Y, exogenous] {$\exovarv{U}$};
      \node (d) [double,left = of ed] {$d$};
      \node (Xd) [left = of ex] {$O_d$};
      \node (YXd) [left = of ey] {$U_{O_d}$};
      \path
      (ed) edge[->] (D)
      (ex) edge[->] (X)
      (ex) edge[->] (Xd)
      (ey) edge[->] (Y)
      (ey) edge[->] (YXd)
      (D) edge[->] (X)
      (d) edge[->] (Xd)
      (X) edge[->] (Y)
      (Xd) edge[->] (YXd)
      (D) edge[->,bend left] (Y)
      ;
      \node (h1) [right = 3mm of Y, inner sep=0mm, minimum size=0mm] {};
      \node (h2) [below = 3mm of Y, inner sep=0mm, minimum size=0mm] {};
      \path (D) edge[out=-55,in=90] (h1)
      (h1) edge[out =-90, in=0] (h2)
      (h2) edge[->, out =180, in=-15] (YXd)
      ;
      \begin{scope}[node distance=0.1cm and 0.1cm,
        every node/.style={ draw=none, rectangle,align=center,
          inner sep=0mm }
        ]
        \scriptsize
        \node [left = of d] {Posts \\ \textcolor{dmblue500}{$d=\text{apolitical}$ }};
        \node [left = of Xd] { Opinion\\
          \textcolor{dmblue500}{$O_d = f_O(d, \exovarv{O})$}};
          \node [left =0mm of YXd] {Clicks\\
          \textcolor{dmblue500}{$U_{O_d} = f_U(D,O_d,\exovarv{U})$}};
        \end{scope}

    \end{influence-diagram}
\caption{SCM with nested counterfactual}\label{fig:counterfactual3}
\end{subfigure}
\begin{influence-diagram}
    \cidlegend[]{
      \legendrow{exogenous}{exogenous\\ node} \\
      \legendrow{}{structural\! node} \\
      \legendrow{double}{intervened\\ node} \\
      \legendrow{decision}{decision node}\\
      \legendrow{utilityc, chamfered rectangle xsep=1.5pt, chamfered rectangle ysep=1.5pt}{utility node}\\
}
\end{influence-diagram}
  \caption{
    An example of a SCIM and interventions.
    In the SCIM, either political or apolitical posts $D$ are displayed.
    These affect the user's opinion $O$. $D$ and $O$ influence the user's clicks $U$ (a).
    Given a policy, the SCIM becomes a SCM (b).
    Interventions and counterfactuals may be defined in terms of this SCM.\@
    For example, the nested counterfactual $U_{O_d}$ represents
    the number of clicks if the user has the opinions that they would
    arrive at, after viewing apolitical content (c).}
  \label{fig:counterfactual}

\end{figure*}

More elaborate hypotheticals can be described with
a nested counterfactual, in which the intervention
is itself a potential response variable.
In \cref{fig:counterfactual3}, the \emph{click} probability $U$ depends on both
the chosen \emph{posts} $D$ and the user \emph{opinion} $O$, which is in turn
also influenced by $D$.
The nested potential response variable $U_{O_d}$, defined by
$U_{O_d}(\exovals) \coloneqq U_{o}(\exovals)$ where $o=O_d(\exovals)$,
represents the probability
that a user clicks on a ``default'' post $D$ given that their opinion has been
influenced by a hypothetical post $d$.
In other words, the effect of the intervention $\doo(D=d)$ is propagated to $U$
only through $O$.

\subsection{Causal Influence Diagrams}

Influence diagrams are graphical models with special decision and utility nodes,
developed to model decision making problems
\citep{howard1990influence,Lauritzen2001}.
Influence diagrams do not in general have causal semantics, although some causal
structure can be inferred \citep{Heckerman1995}.
We will assume that the edges of the influence diagram reflect the causal
structure of the environment, so we use the term “Causal Influence Diagram”.

\begin{definition}[Causal influence diagram]\label{def:cid}
    A \emph{causal influence diagram} (CID) is a directed acyclic graph
$\cid$ where the vertex set $\evars$ is partitioned into \emph{structure
nodes} $\structvars$, \emph{decision nodes} $\decisionvars$, and \emph{utility
nodes} $\utilvars$.
Utility nodes have no children.~\looseness=-1
\end{definition}

We use $\Pav{\evar}$ and $\Descv{\evar}$ to denote the
parents and descendants of a node $\evar \in \evars$.
The parents of the decision, $\Pad$, are also called observations.
An edge from node $V$ to node $Y$ is denoted $V \to Y$.
Edges into decisions are called information links, as they indicate what
information is available at the time of the decision.
A directed path (of length at least zero) is denoted $V \pathto Y$.
For sets of variables, $\sV \pathto \sY$ means that $V \pathto Y$ holds for some
$V \in \sV$, $Y \in \sY$.~\looseness=-1

\subsection{Structural Causal Influence Models}
For our new incentive concepts, we define a hybrid of the influence diagram and the SCM.\@
Such a model, originally proposed by \citet{dawid2002influence}, has
structure and utility nodes with associated functions,
exogenous variables with an associated probability distributions,
and decision nodes, without any function at all, until one is selected by an agent.\footnote{Dawid called this a ``functional influence diagram''.
  We favour the term SCIM, because the corresponding term SCM
  is more prevalent than ``functional model''.
}
This can be formalised as the \emph{structural causal influence model} (SCIM, pronounced `skim').

\begin{definition}[Structural causal influence model]\label{def:scim}

    A \emph{structural causal influence model} (SCIM)
is a tuple $\scim = \scimdef$ where:
\begin{itemize}
  \item $\cid$ is a CID with finite-domain variables $\evars$
    (partitioned into $\structvars$, $\decisionvars$, and $\utilvars$)
where utility variable domains are a subset of $\Reals$.
    We say that $\scim$ is \emph{compatible with} $\cid$.
  \item $\exovars = \setfor{\exovarv{\evar}}{\evar \in \evars}$
    is a set of finite-domain \emph{exogenous variables}, one for each
    endogenous variable.
  \item
    $\structfns \!=\! \setfor{\fv{\evar}}{\evar \in \evars \setminus \decisionvars}$
    is a set of \emph{structural functions}
$\fv{\evar}\colon {\dom(\Pav{\evar} \cup \{\exovarv{\evar}\}) \to \dom({\evar})}$
    that specify how each
    non-decision endogenous variable depends on its parents in
    $\cid$ and its associated exogenous variable.
  \item $\exoprob$ is a probability distribution for $\exovars$
    such that the individual exogenous variables $\exovarv{\evar}$ are mutually
    independent.
\end{itemize}
\end{definition}

We will restrict our attention to single-decision settings with
$\decisionvars=\{D\}$.
An example of such a SCIM for the content recommendation example
is shown in \cref{fig:scim-example}.
In single-decision SCIMs,
the decision-making task is to maximize expected utility by selecting
a decision $d\in \dom(D)$ based on the observations $\Pad$.
More formally, the task is to select a structural function for $D$ in the form
of a \emph{policy} $\pi: \sfsig{\decisionvar}$.
The exogenous variable $\exovarv{D}$ provides randomness to allow
the policy to be a stochastic function of its endogenous parents $\Pad$.
The specification of a policy turns a SCIM $\scim$ into an SCM
$\scim_\pi := \langle \exovars, \evars, \structfns \cup \{\pi\}, \exoprob
\rangle$, see \cref{fig:counterfactual1}.
With the resulting SCM, the standard definitions of causal interventions apply.
Note that what determines whether a node is observed or not at the time
of decision-making is whether the node is a parent of the decision.
Commonly,
some structure nodes represent latent variables that are unobserved.

We use $\Prs{\pi}$ and $\EEs{\pi}$ to denote probabilities and expectations
with respect to $\scims{\pi}$.
For a set of variables $\sX$ not in $\Descv{\decisionvar}$,
$\Prs{\pi}(\sx)$ is independent of $\pi$ and we simply write $\Pr(\sx)$.
An \emph{optimal policy} for a SCIM is defined as any policy $\pi$
that maximises $\EEs{\pi}[\totutilvar]$, where
$\totutilvar \coloneqq \sum_{\utilvar \in \utilvars}{\utilvar}$.
A potential response $\totutilvar_\sx$ is defined as
$\totutilvar_\sx \coloneqq \sum_{\utilvar \in \utilvars}{\utilvar_\sx}$.~\looseness=-1

\section{Materiality}

Next, we review a characterization
of which observations are \emph{material} for optimal performance,
as this will be a fundamental building block for most of our theory.\footnote{In contrast to subsequent sections, the results in this section and
  the VoI section do not require the influence diagrams to be causal.}~\looseness=-1

\begin{definition}[Materiality; \citealp{Shachter2016}]
  For any given SCIM $\scim$, let
  $\attutil(\scim)=\max_{\pi}\EEs{\pi}[\totutilvar]$ be the maximum attainable
  utility in $\scim$, and let $\scim_{X \not \to D}$ be $\scim$
  modified by removing any information link $X \to D$.  
  The observation $X\in \Pad$ is \emph{material} if
  $\attutil(\scim_{X \not \to D})
  <
  \attutil(\scim)
  $.
\end{definition}

Nodes may often be identified as immaterial
based on the graphical structure alone \citep{fagiuoli1998note,Lauritzen2001,Shachter2016}.
The graphical criterion uses uses the notion of d-separation.~\looseness=-1

\begin{definition}[d-separation; \citealp{Verma1988soundness}]\label{def:d-separation}
    A path $p$ is said to be d-separated by a set of nodes $\sZ$ if and only if:
    \begin{enumerate}
    \item $p$ contains a collider $X \to W \gets Y$, such that the middle node $W$ is not in $\sZ$ and no descendants of $W$ are in $\sZ$, or
    \item $p$ contains a chain $X \to W \to Y$ or fork $X \gets W \to Y$ where $W$ is in $\sZ$, or
    \item one or both of the endpoints of $p$ is in $\sZ$.
    \end{enumerate}
    A set $\sZ$ is said to d-separate $\sX$ from $\sY$,
written ${(\sX \dsepg \sY \mid \sZ)}$ if and only if $\sZ$ d-separates every path
from a node in $\sX$ to a node in $\sY$. Sets that are not d-separated are
called d-connected.~\looseness=-1
\end{definition}
 
According to the graphical criterion of \citet{fagiuoli1998note}, an observation
cannot provide useful information if it is d-separated from utility, conditional on
other observations.
This condition is called nonrequisiteness.

\begin{definition}[Nonrequisite observation; \citealp{Lauritzen2001}]
  \label{def:requisite-observation}
  Let $\utilvarsd := \utilvars\cap\Descv{\decisionvar}$ be
  the utility nodes downstream of $\decisionvar$.
  An observation $X\in \Pad$ in a single-decision CID $\causalgraph$ is
  \emph{nonrequisite} if:  \begin{equation}
    \label{eq:voi-criterion}
    X \perp \utilvars^D \bmid \left(\Pad \cup \{\decisionvar\} \setminus \{X\}\right)
  \end{equation}
  In this case, the edge $X\to \decisionvar$ is also called nonrequisite.
  Otherwise $X$ and $X\to \decisionvar$ are \emph{requisite}.
\end{definition}
 
For example, in \cref{fig:race-a}, \emph{high school} is a requisite observation while
\emph{gender} is not.

\section{Value of Information}
Materiality can be generalized to nodes not observed,
to assess which variables a decision-maker would benefit from
knowing before making a decision, i.e.\ which variables have VoI \citep{Howard1966,Matheson1990}.
To assess VoI for a variable $X$,
we first make $X$ an observation by adding a link $X\to D$,
and then test whether $X$ is material in the updated model \citep{Shachter2016}.

\begin{definition}[Value of information] \label{def:observation-incentive-sa}
  A node $X \in \evars \setminus \Descv{\decisionvar}$ in a
  single-decision SCIM $\scim$ has \emph{VoI}
  if it is material in the model $\scim_{X \to D}$ obtained
  by adding the edge $X \to D$ to $\scim$.
  A CID $\causalgraph$ \emph{admits VoI} for $X$ if $X$ has VoI in a
  a SCIM $\scim$ compatible with $\causalgraph$.
\end{definition}

Since \cref{def:observation-incentive-sa} adds an information link, it can only be applied to
non-descendants of the decision, lest cycles be created in the graph.
Fortunately, the
structural functions need not be adapted for the added link,
since there is no structural function associated with $D$.

We prove that the graphical criterion of \cref{def:requisite-observation} is
tight for both materiality and VoI, in that it
identifies every zero VoI node that can be identified from the graphical
structure (in a single decision setting).

\begin{theorem}[Value of information criterion]
  \label{th:observation-sa}
  A single decision CID $\causalgraph$ admits VoI for $X \in \evars \setminus \Descv{\decisionvar}$
    if and only if $X$ is a requisite observation in $\causalgraph_{X \to D}$, the graph 
    obtained by adding $X \to D$ to $\causalgraph$.
\end{theorem}

The soundness direction (i.e.\ the \emph{only if} direction) follows from d-separation \citep{fagiuoli1998note,Lauritzen2001,Shachter2016}.
In contrast, the completeness direction does not follow from the
completeness property of d-separation.
The d-connectedness of $X$ to $\utilvars$ implies that $\utilvars$ may be conditionally dependent on $X$.
It does not imply, however, that the expectation of $\utilvars$ or the utility attainable under an optimal policy will change.
Instead, our proof (\cref{sec:appendix-voi}) constructs a SCIM such that $X$ is material.
This differs 
from a previous attempt by \citet{nielsen1999welldefined},
as discussed in Related Work.

We apply the graphical criterion to the grade prediction example in \cref{fig:race-a}.
One can see that the predictor has an incentive to use the incoming student's high
school but not gender.
This makes intuitive sense, given
that \emph{gender} provides no information useful for predicting the university grade in this example.

 \section{Response Incentives}\label{sec:response}
There are two ways to understand a material observation. One is that it provides useful information. From this perspective, a natural generalisation is VoI, as described in the previous section. An alternative perspective is that a material observation is one that influences optimal decisions. Under this interpretation, the natural generalisation is the set of all (observed and unobserved) variables that influence the decision. We say that these variables have a response incentive.\footnote{The term \emph{responsiveness} \citep{Heckerman1995,Shachter2016}
  has a related but not identical meaning --
it refers to whether a decision $D$ affects a variable $X$
rather than whether $X$ affects $D$.}

\begin{definition}[Response incentive]\label{def:response-incentive}
Let $\scim$ be a single-decision SCIM.\@
A policy $\pi$ \emph{responds} to a variable $X\in \rivars$ if
there exists some intervention $\Do(\incentivevar = \incentiveval)$
and some setting $\exovars = \exovals$, such that
$\decisionvar_\incentiveval(\exovals) \ne \decisionvar(\exovals)$.
The variable
$\incentivevar$ has a \emph{response incentive} if all optimal
policies respond to $X$.

A CID \emph{admits} a response incentive on $\incentivevar$
if it is compatible with a SCIM that has a response incentive on
$\incentivevar$.
\end{definition}

For a response incentive on $X$ to be possible, there must be:
i) a directed path $X \pathto D$, and ii) an incentive for $D$
to use information from that path.
For example, in \cref{fig:race-a}, \emph{gender} has a directed path to the decision
but it does not provide any information about the likely grade, so there is no response incentive.
The graphical criterion for RI builds on a modified graph with nonrequisite
information links removed.

\begin{definition}[Minimal reduction; \citealp{Lauritzen2001}]
  \label{def:reduced-graph}
  The \emph{minimal reduction} $\reducedgraph$ of a single-decision CID $\causalgraph$
  is the result of removing from $\causalgraph$ all information links from nonrequisite observations.
\end{definition}

The presence (or absence) of a path $X \pathto D$ in the minimal reduction tells us whether
a response incentive can occur.~\looseness=-1

\begin{theorem}[Response incentive criterion]\label{theorem:ri-graph-criterion}
A single-decision CID $\cid$ admits a response
incentive on $\incentivevar \in \rivars$ if and only if 
the minimal reduction $\reducedgraph$ has a 
directed path 
$\incentivevar \pathto \decisionvar$.
\end{theorem}
 \enlargethispage{\baselineskip}
\begin{proof}
  The \emph{if} (completeness) direction is proved in
  \cref{theorem:ri-graph-criterion-completeness} in \cref{sec:appendix-ri}.
  For the soundness direction,
assume that for $\cid$, the minimal reduction $\rcid$ does not contain a directed path $\incentivevar \pathto \decisionvar$.
Let $\scim = \scimdef$ be any SCIM compatible with $\cid$.
Let $\rscim = \left\langle \rcid,\exovars,\structfns,\exoprob \right\rangle$
be $\scim$, but with the minimal reduction $\rcid$.
By \cref{le:reduced-optimal-policy} in \cref{sec:proofs}, there exists a $\rcid$-respecting policy $\tilde \pi$ that is optimal in $\scim$.
In $\rscim_{\tilde \pi}$,
$\incentivevar$ is causally irrelevant for $\decisionvar$ so $\decisionvar(\exovals) = \decisionvar_x(\exovals)$.
Furthermore, $\scim_{\tilde \pi}$ and $\rscim_{\tilde \pi}$ are the same SCM,
with the functions $\structfns \cup \{\tilde \pi\}$.
So $\decisionvar(\exovals) = \decisionvar_x(\exovals)$ also in
$\scim_{\tilde\pi}$, which means that there is an optimal policy in $\scim$ that
does not respond to interventions on $X$ for any $\exovals$.
\end{proof}

The intuition behind the proof is that an optimal decision only responds to
effects that propagate to one of its requisite observations.
For the completeness direction, we show in \cref{sec:appendix-ri}
that if $X \pathto D$ is present in the minimal reduction $\reducedgraph$, then
we can select a SCIM $\scim$ compatible with $\cid$ such that $D$ receives useful
information along that path, that any optimal policy must respond to.

In a safety setting, it may be desirable for an AI system to have an incentive to respond to its shutdown button,
so that when asked to shut down, it does so \citep{Hadfield-Menell2016osg}.
In a fairness setting, on the other hand, a response incentive may be a cause
for concern, as illustrated next.~\looseness=-1

\begin{figure}
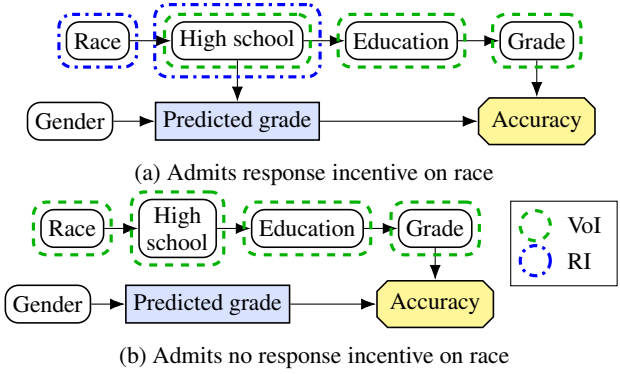

\begin{subfigure}[t]{\linewidth}
    \centering
      \begin{influence-diagram}
    \setrectangularnodes
    \setcompactsize
    \tikzset{node distance=5.5mm and 5.5mm}

\node (R) [] {Race};
      \node (S) [right= of R] {High school};
      \node (E) [right= of S] {Education};
      \node (Gr) [right=of E] {Grade};
\node (D) [below=of S,decision] {Predicted grade};
      \node (Ge) [left= of D] {Gender};
      
      \node (U) at (Gr|-D) [utility] {Accuracy};

    \draw[->]
      (R) edge (S)
      (S) edge (E)
      (S) edge (D)
      (E) edge (Gr)
      (Gr) edge (U)
      (Ge) edge (D)
      (D) edge (U)
    ;

\observationincentive{S}
    \observationincentive{E}
    \observationincentive{Gr}
\responseincentive{R}
    \responseincentive[inner sep=2.2mm]{S}

\end{influence-diagram}

 \caption{
        Admits response incentive on race
    }\label{fig:race-a}
\end{subfigure}
  \begin{subfigure}[t]{\linewidth}
    \centering
      \begin{influence-diagram}
    \setrectangularnodes
    \setcompactsize
    \tikzset{node distance=3.5mm and 4.5mm}

    \node (R) [] {Race};
      \node (S) [right= of R] {High\\ school};
      \node (E) [right= of S] {Education};
      \node (Gr) [right=of E] {Grade};
\node (D) [below=of S, xshift=4mm, decision] {Predicted grade};
      \node (Ge) [left= of D] {Gender};
      
      \node (U) at (Gr|-D) [utility] {Accuracy};

    \draw[->]
      (R) edge (S)
      (S) edge (E)
(E) edge (Gr)
      (Gr) edge (U)
      (Ge) edge (D)
      (D) edge (U)
;

\observationincentive{R}
\observationincentive{S}
    \observationincentive{E}
    \observationincentive{Gr}

    \cidlegend[right = 5mm of Gr, yshift=-2mm]{
      \legendrow{observation incentive}{VoI} \\
      \legendrow{response incentive}{RI} \\
}
  \end{influence-diagram}

     \caption{
      Admits no response incentive on race
    }\label{fig:race-b}
  \end{subfigure}
  \caption{In (a), the admissible incentives of the grade prediction example from \cref{fig:race-preview} are shown,
      including a response incentive on race.
      In (b), the predictor no-longer has access to the students' high school,
      and hence there can no-longer be any response incentive on race.~\looseness=-1
  }\label{fig:race}
\end{figure}
 \paragraph{Incentivised unfairness}
Response incentives are closely related to counterfactual fairness \citep{kusner2017counterfactual,kilbertus2017discrimination}.
A prediction --- or more generally a decision --- is considered counterfactually
unfair if a change to a \emph{sensitive attribute} like race or gender would
change the decision.

\begin{definition}[Counterfactual fairness; {\citealp{kusner2017counterfactual}}]
A policy $\pi$ is \emph{counterfactually fair} with respect to a
  sensitive attribute $A$ if
\begin{equation*}
    \label{eq:cf}
    \Prs{\pi}\left(
      \decisionvar_{a'} = \decisionval\mid \pad, a \right) =
    \Prs{\pi}\left(
      \decisionvar = \decisionval \mid \pad, a\right)
  \end{equation*}
  for every decision $\decisionval \in \dom(\decisionvar)$,
  every context $\pad \in \dom(\Pad)$, and
  every pair of attributes $a, a' \in \dom(A)$ with $\Pr(\pad, a) > 0$.
\end{definition}
 
A response incentive on a sensitive attribute indicates that counterfactual unfairness is incentivised,
as it implies that \emph{all} optimal policies are counterfactually unfair:

\begin{restatable}[Counterfactual fairness and response incentives]{theorem}{theoremcffair}\label{theorem:counterfactual-fairness}
In a single-decision SCIM $\scim$ with a sensitive attribute $A\in\rivars$,
all optimal policies $\pi^*$ are counterfactually unfair
with respect to $A$ if and
only if $A$ has a response incentive.
\end{restatable}

 The proof is given in \cref{sec:appendix-fairness}.

A response incentive on a sensitive attribute means
that counterfactual unfairness is not just possible, but incentivised.
As a result, it has a more restrictive graphical criterion.
The graphical criterion for counterfactual fairness states that a decision can only be counterfactually
unfair with respect to a sensitive attribute if that attribute is an ancestor of the decision \citep[Lemma 1]{kusner2017counterfactual}.
For example, in the grade prediction example of \cref{fig:race-a},
it is possible for a predictor to be counterfactually unfair
with respect to either \emph{gender} or \emph{race}, because both are ancestors of the decision.
The response incentive criterion can tell us in which case counterfactual unfairness is actually incentivised.
In this example, the minimal reduction includes the edge from
\emph{high school} to \emph{predicted grade} and hence the directed path from \emph{race} to \emph{predicted grade}.
However, it excludes the edge from \emph{gender} to \emph{predicted grade}.
This means that the agent is incentivised to be counterfactually unfair with respect to
\emph{race} but not to \emph{gender}.~\looseness=-1

Based on this, how should the system be redesigned? According to the response incentive criterion,
the most important change is to remove the path from
\emph{race} to \emph{predicted grade} in the minimal reduction. This
can be done by removing the agent's access to \emph{high school}.
This change is implemented in \cref{fig:race-b}, where there is no
response incentive on either
sensitive variable.~\looseness=-1

Value of information is also related to fairness.
For a sensitive variable that is not a parent of the decision, positive VoI means that \emph{if} the predictor gained access to its value,
then the predictor would use it.
For example, if in \cref{fig:race-b} an edge is added from \emph{race} to \emph{predicted grade}, then unfair behaviour will result.
In practice, such access can result from unanticipated correlations between
the sensitive attribute and parents of the decision, rather
than the system being given direct access to the attribute.
Analysing VoI may help detect such problems at an early stage.
However, VoI is less closely related to counterfactual fairness than response incentives.
In particular, \emph{race} lacks VoI in \cref{fig:race-a}, but counterfactual unfairness is incentivised.
On the other hand, \cref{fig:race-b} admits positive VoI for \emph{race}, but counterfactual unfairness is not incentivised.

The incentive approach is not restricted to counterfactual fairness.
For any fairness definition, one could assess whether that kind of unfairness is incentivised
by checking whether it is present under all optimal policies.

 \section{Value of Control}
\label{sec:control-incentives}
\enlargethispage{\baselineskip}

A variable has VoC if a decision-maker could benefit from setting its value \citep{shachter1986evaluating,Matheson1990,Shachter2010}.
Concretely, we ask whether the attainable utility can be increased by
letting the agent decide the structural function for the variable.

\begin{definition}[Value of control]\label{def:control-incentive-sa}
  In a single-decision SCIM $\scim$, a non-decision node
    $X$ has \emph{positive value of control}
  if
  \[
    \max_{\pi}\EEs{\pi}[\totutilvar]
    <
    \max_{\pi, \gx}\EEs{\pi}[\totutilvar_{\gx}]
    \]
    where $\gx:\sfsig{X}$ is a soft intervention at $X$,
    i.e.\ a new structural function for $X$ that respects the graph.
\end{definition}

A CID $\causalgraph$ \emph{admits positive value of control} for $\incentivevar$ if
there exists a SCIM $\scim$ compatible with $\causalgraph$ where $\incentivevar$ has positive value of control.
This can be deduced from the graph, using again the minimal reduction
(\cref{def:reduced-graph}) to rule out effects through observations that an
optimal policy can ignore.~\looseness=-1

\begin{theorem}[Value of control criterion]
  \label{th:soft-sa}
  A single-decision CID $\causalgraph$ admits positive value of control for a
  node $X\in\evars \setminus \{D\}$
  if and only if there is a directed path
  $X\pathto \utilvars$ in the minimal reduction $\reducedgraph$.
\end{theorem}

\begin{proof}
    The \emph{if} (completeness) direction is proved in \cref{le:voc-completeness}.
    The proof of \emph{only if} (soundness) is as follows.
Let $\scim = \scimdef$ be a single-decision SCIM.\@
    Let $\scims{\gx}$ be $\scim$, but with the structural function $\fv{X}$
    replaced with $\gx$.
    Let $\rscim$ and $\rscim_{\gx}$ be the same SCIMs, respectively, but
    replacing each graph with the minimal reduction $\reducedgraph$.

    Recall that $ \EEs{\pi}[\totutilvar_{\gx}]$ is defined by applying the soft
    intervention $\gx$ to the (policy-completed) SCM $\scims{\pi}$.
    However, this is equivalent to applying the policy $\pi$ to the modified
    SCIM $\scims{\gx}$, as the resulting SCMs are identical.
    Since $\scims{\gx}$ is a SCIM, \cref{le:reduced-optimal-policy} can be applied,
    to find a $\reducedgraph$-respecting optimal policy $\tilde\pi$ for
    $\scims{\gx}$.

    Consider now the expected utility under an arbitrary intervention $\gx$
    for a policy $\pi$ optimal for $\scims{\gx}$:
\begin{align*}
&\EEs{\pi}[\totutilvar_{\gx}] \text{ in $\scim$} \\
      &=\EEs{\pi}[\totutilvar] \text{ in $\scims{\gx}$} & \text{by SCM equivalence} \\
      &=\EEs{\tilde\pi}[\totutilvar] \text{ in $\scims{\gx}$} & \text{by \cref{le:reduced-optimal-policy}}\\
      &=\EEs{\tilde\pi}[\totutilvar] \text{ in $\rscim_{\gx}$} & \text{since $\tilde\pi$ is $\reducedgraph$-respecting} \\
        &=\EEs{\tilde\pi}[\totutilvar] \text{ in $\rscim$} & \text{by \cref{le:sigma-calculus-intervention}} \\
       &=\EEs{\tilde\pi}[\totutilvar] \text{ in $\scim$} & \text{only increasing the policy set} \\
       &\leq \max_{\pi^*}\EEs{\pi^*}[\totutilvar] \text{ in $\scim$} & \text{$\max$ dominates all elements.}
    \end{align*}
    This shows that $X$ must lack value of control.
\end{proof}

The proof of the completeness direction (\cref{sec:appendix-voc}) establishes
that if a path exists,
then a SCIM be selected where the intervention on $X$ can either directly control $U$ or increase the useful information available at $D$.

To apply this criterion to the content recommendation example (\cref{fig:fci-application1}), we first obtain the minimal reduction, which is identical
to the original graph.
Since all non-decision nodes are upstream of the utility in the minimal reduction, they all admit positive VoC.
Notably, this includes nodes like \emph{original user opinions} and \emph{model of user opinions}
that the decision has no ability to control according to the graphical structure.
In the next section, we propose \emph{instrumental control incentives}, which
incorporate the agent's limitations.~\looseness=-1

 \section{Instrumental Control Incentive}\label{sec:fci}

Would an agent use its decision to control a variable $X$?
This question has two parts: whether $X$ is useful to control (VoC),
and whether $X$ is possible to control (responsiveness).
As described in the previous section, VoC uses $\totutilvar_{\gx}$ to consider
the utility attainable from arbitrary control of $X$.
Meanwhile, $X_d$ describes the way $X$ can be controlled by $D$.
These notions can be combined with a nested counterfactual $\totutilvar_{X_d}$,
which expresses the effect that $D$ can have on $\totutilvar$ by controlling $X$.

\begin{definition}[Instrumental control incentive]\label{def:instrumental-goal}
In a single-decision SCIM $\scim$,
there is an \emph{instrumental control incentive} on a variable $X$ in decision context
$\pad$ if, for all optimal policies $\pi^*$,
\begin{equation}\label{eq:ci}
\EEs{\pi^*}{[\totutilvar_{\incentivevar_{\decisionval}} \mid \pad]}
\neq
\EEs{\pi^*}{[\totutilvar \mid \pad]}.
\end{equation}
\end{definition}

Conceptually, an instrumental control incentive can be interpreted as
follows.
If the agent got to choose $D$ to influence $X$ independently of how $D$
influences other aspects of the environment, would that choice matter?
We call it an \emph{instrumental} control incentive, as the control of $X$
is a tool for achieving utility
(cf.\ \emph{instrumental goals}
\citealp{omohundro2008aidrives,Bostrom2014}).
ICIs do not consider side-effects of the optimal policy:
for instance, it may be that all optimal policies affect $X$ in a particular
way, even if $X$ is a not an ancestor of any utility node ---
in such cases, no ICI is present.
Finally, in \posscite{pearl2001direct} terminology, an instrumental
control incentive corresponds to a \emph{natural indirect effect} from
$D$ to $U$ via $X$ in $\scim_{\pi^*}$, for all optimal policies $\pi^*$.

A CID $\cid$ \emph{admits} an instrumental control incentive on
$\incentivevar$ if $\cid$ is compatible with a SCIM $\scim$ with
an instrumental control incentive on $X$ for some decision context $\pad$.
The following theorem gives a sound and complete graphical criterion for
which CIDs admit instrumental control incentives.

\begin{theorem}[Instrumental Control Incentive Criterion]\label{theorem:ig-graph-criterion}
A single-decision CID $\cid$ admits an instrumental control incentive on
$\incentivevar \in \civars$ if and only if $\cid$ has a directed
path from the decision $\decisionvar$ to a utility node $\utilvar \in \utilvars$
that passes through $\incentivevar$,
i.e.\ a directed path $\decisionvar \pathto \incentivevar \pathto \utilvars$.
\end{theorem}
 
\begin{proof}Completeness (the \emph{if} direction) is proved in \cref{le:fci-completeness}.
    The proof of soundness is as follows.

Let $\scim$ be any SCIM compatible with $\cid$ and $\pi$ any policy for $\scim$.
We consider variables in the SCM $\scims{\pi}$.
If there is no directed path
$\decisionvar \pathto \incentivevar \pathto \utilvars$ in $\cid$,
then either
${\decisionvar \nopathto \incentivevar}$ or
${\incentivevar \nopathto \utilvars}$.
If ${\decisionvar \nopathto \incentivevar}$, then
${\incentivevar_{\decisionval}(\exovals)} =
    {\incentivevar(\exovals)}$
    for any setting $\exovals \in \dom(\exovars)$ and decision $\decisionval$
(\cref{theorem:path-causal-irrelevance}).
Therefore, ${\totutilvar(\exovals)}
    = {\totutilvar_{\incentivevar_{\decisionval}}(\exovals)}$.
Similarly, if ${\incentivevar \nopathto \utilvars}$ then
    ${\utilvar(\exovals)} =
    {\utilvar_\incentiveval(\exovals)}$ for every
    setting $\exovals \in \dom(\exovars)$,
$\incentiveval \in \dom(\incentivevar)$ and $\utilvar \in \utilvars$
    so ${\totutilvar(\exovals)} =
    {\totutilvar_{\incentivevar_{\decisionval}}(\exovals)}$.
In either case,
$\EEs{\pi}[\totutilvar \mid \pad]
= \EEs{\pi}[\totutilvar_{\incentivevar_\decisionval} \mid \pad]$
and there is no instrumental control incentive on $\incentivevar$.
\end{proof}

The logic behind the soundness proof above is that if there is no path from $D$ to $X$ to
$\utilvars$, then $D$ cannot have any effect on $\utilvars$ via $X$.
For the completeness direction proved in \cref{sec:appendix-fci},
we show how to construct a SCIM
so that $U_{X_d}$ differs from the non-intervened $U$ for any diagram with a
path $D\pathto X\pathto \utilvars$.

\iftwocol
  \newcommand{\casubfigwidth}{\columnwidth}
\else
  \newcommand{\casubfigwidth}{0.475\textwidth}
\fi

\begin{figure}
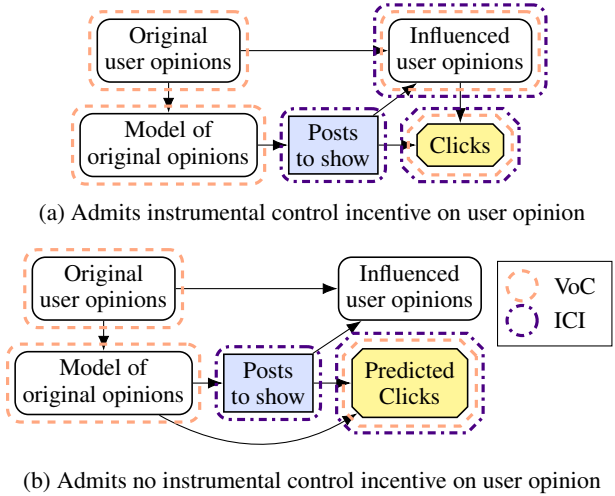

\begin{subfigure}[t]{\linewidth}
    \centering
    \begin{influence-diagram}
  \setrectangularnodes
  \setcompactsize

  \node (D) [decision] {Posts\\ to show};
  \node (M) [left = of D] {Model of \\ original opinions};
  \node (P1) [above =4mm of M] {Original\\ user opinions};
  \node (U) [right =5mm of D, utility] {Clicks};
  \node (P2) at (U|-P1) {Influenced\\ user opinions};

  \draw[->]
    (P1) edge (M)
    (M) edge[information] (D)
    (P1) edge (P2)
    (D) edge (P2)
    (D) edge (U)
    (P2) edge (U)
  ;

  \interventionincentive{P1}
  \interventionincentive{M}
  \interventionincentive{P2}
  \interventionincentive[uchamf]{U}
  \controlincentive{D}
  \controlincentive[inner sep=1.8mm]{P2}
  \controlincentive[uchamf,inner sep=0.9mm]{U}
\end{influence-diagram}
     \caption{
        Admits instrumental control incentive on user opinion
    }
    \label{fig:fci-application1}
  \end{subfigure}
  \vspace{2mm}

  \begin{subfigure}[t]{\linewidth}
  \centering
  \begin{influence-diagram}
  \setrectangularnodes
  \setcompactsize

  \node (D) [decision] {Posts\\ to show};
  \node (M) [left = of D] {Model of \\ original opinions};
  \node (P1) [above =4mm of M] {Original\\ user opinions};
  \node (U) [right =5mm of D, utility] {Predicted\\Clicks};
  \node (P2) at (U|-P1) {Influenced\\ user opinions};

  \draw[->]
    (P1) edge (M)
    (M) edge (D)
    (P1) edge (P2)
    (D) edge (P2)
    (D) edge (U)
  ;
  \path (M) edge[->, bend right] (U);

  \interventionincentive{P1}
  \interventionincentive{M}
  \controlincentive{D}
  \controlincentive[uchamf,inner sep=0.9mm]{U}
  \interventionincentive[uchamf]{U}

  \cidlegend[right = 2mm of P2, yshift=-2mm]{
      \legendrow{intervention incentive}{VoC} \\
      \legendrow{control incentive}{ICI} \\
}

\end{influence-diagram}

       \caption{
      Admits no instrumental control incentive on user opinion
  }
  \label{fig:fci-application2}
  \end{subfigure}
  \caption{In (a), the content recommendation example from \cref{fig:fci-preview} is shown to
      admit an instrumental control incentive on user opinion.
      This is avoided in (b) with a change to the objective.
      }\label{fig:fci-application}
\end{figure}
 
\pagebreak
Let us apply this criterion to the content recommendation example
in \cref{fig:fci-application1}.
The only nodes $X$ in this graph that lie on a path $D \pathto X \pathto \utilvars$
are \emph{clicks} and \emph{influenced user opinions}.
Since \emph{influenced user opinions} has an instrumental control incentive,
the agent may seek to influence that variable in order to attain utility.
For example, it may be easier to predict what content a more emotional user will
click on and therefore, a recommender may
achieve a higher click rate by introducing posts that
induce strong emotions.~\looseness=-1

How could we instead design the agent to maximise clicks without manipulating the user's opinions
(i.e.\ without an instrumental control incentive on \emph{influenced user opinions})?
As shown in \cref{fig:fci-application2}, we could redesign the system so that
instead of being rewarded for the true click rate, it is rewarded for the
clicks it would be predicted to have, based on a separately trained
model of the user's preferences.
An agent trained in this way would view any modification of user opinions as
irrelevant for improving its performance; however, it would still have an
instrumental control incentive for \emph{predicted clicks} so it would still deliver
desired content.
To avoid undesirable behaviour in practice, the click prediction must
truly predict whether the original user would click the content, rather than
baking in the effect of changes to the user's opinion from reading earlier posts.
This could be accomplished, for instance, by training a model to predict
how many clicks each post would receive if it was offered individually.~\looseness=-1

This dynamic is related to concerns about the long-term safety of AI systems.
For example, \citet{russell2019human} has hypothesised that an advanced AI system would seek to manipulate
its objective function (or human overseer) to obtain reward.
This can be understood as an instrumental control incentive on the objective function (or the overseer's behaviour).
A better understanding of incentives could therefore be relevant for designing safe systems in both the short and long-term.

 \pagebreak
\section{Related Work} \label{sec:related-work}

\paragraph{Causal influence diagrams}
\citet{Jern2011} and \citet{Kleiman-Weiner2015} define influence diagrams with causal edges,
and similarly use them to model decision-making of rational agents
(although they are less formal than us, and focus on human decision-making).

An informal precursor of the SCIM that also used structural functions (as
opposed to conditional probability distributions)
was the ``functional influence diagram'' \citep{dawid2002influence}.
The most similar alternative model is the Howard canonical form influence diagram \citep{howard1990influence,Heckerman1995}.
However, this only permits counterfactual reasoning downstream of decisions, which is
inadequate for defining the response incentive.
Similarly, the causality property for influence diagrams introduced by
\citet{Heckerman1994} and \citet{Shachter2010} only constrains the relationships to be partially causal
downstream of the decision
(though adding new decision-node parents to all nodes makes the diagram fully causal).
\Cref{app:causality-examples} shows by example why the stronger causality
property
is necessary for most of our incentive concepts.
~\looseness=-1

An open-source Python implementation of CIDs has recently been
developed\footnote{\url{https://github.com/causalincentives/pycid}}
\citep{Fox2021}.

\paragraph{Value of information and control} \Cref{th:soft-sa,th:observation-sa}
for value of information and value of control  build on previous work.
The concepts were first introduced by \citet{Howard1966} and
\citet{shachter1986evaluating}, respectively.
The VoI soundness proof follows previous proofs \citep{Shachter1998,Lauritzen2001},
while the VoI completeness proof is most similar to an attempted proof
by \citet{nielsen1999welldefined}. They propose the criterion $X \not\perp \utilvars^D\mid \Pa_D$ for requisite
nodes, which differs from \cref{eq:voi-criterion} in the conditioned set.
Taken literally,\footnote{Def.~\ref{def:d-separation} defines d-separation for potentially
overlapping sets.}
their criterion is unsound for requisite nodes and positive VoI.
For example, in \cref{fig:race-a}, \emph{High school} is d-separated from
\emph{accuracy} given $\Pad$, so their criterion would
fail to detect that \emph{High school} is requisite and admits VoI.\footnote{Furthermore, to prove that nodes meeting the d-connectedness property are requisite,
\citeauthor{nielsen1999welldefined} claim that
``$X$ is [requisite] for $D$ if $\Pr(\dom(U)\mid D,\Pad)$ is a function of $X$ and $U$ is a utility function relevant for $D$''.
However, $U$ being a function of $X$ only proves that $U$ is conditionally dependent on $X$,
not that it changes the expected utility, or is requisite or material.
Additional argumentation is needed to show that conditioning on $X$ can actually change the expected utility;
our proof provides such an argument.\looseness=-1
}\footnote{Since a preprint of this paper was placed online \citep{everitt2019understanding},
this completeness result was independently discovered by
\citet[Thm. 2]{zhang2020causal} and \citet[Thm. 1]{lee2020characterizing}.
Theorem 2 in the latter also provides a criterion for material observations in a multi-decision
setting.
}

To have positive VoC, it is known that a node must be an ancestor of a value node \citep{shachter1986evaluating},
but the authors know of no more-specific criterion.
The concept of a \emph{relevant} node introduced by
\citet{nielsen1999welldefined} also bears some semblance to VoC.

The relation of the current technical results to prior work is summarised in
Table S1 in the Appendix.

\paragraph{Instrumental control incentives}
\citet{Kleiman-Weiner2015} use (causal) influence diagrams to define a notion
of \emph{intention}, that captures which nodes an optimal policy seeks to
influence.
Intention is conceptually similar to instrumental control incentives and uses
hypothetical node deletions to ask which nodes the agent intends to control.
Their concept is more refined than ICI in the sense that it includes includes
only the nodes that determine optimal policy behaviour, but the definition is
not properly formalized and it is not clear that it can be applied to all
influence diagram structures.

\paragraph{AI fairness}
Another application of this work is to evaluate when an AI system is incentivised to behave unfairly,
on some definition of fairness.
Response incentives address this question for counterfactual fairness \citep{kusner2017counterfactual,kilbertus2017discrimination}.
An incentive criterion corresponding to path-specific effects \citep{zhang2016causal,Nabi2018} is deferred to future work.
\citet{nabi2019learning} have shown how a policy may be chosen subject to path-specific effect constraints.
However, they assume recall of all past events, whereas the response incentive criterion applies to any CID.~\looseness=-1

\paragraph{Mechanism design}
The aim of mechanism design is to understand how objectives and environments can
be designed, in order to shape the behavior of rational agents (e.g.\ \citealp[Part II]{Nisan2007}).
At this high level, mechanism design is closely related to the incentive
design results we have developed in this paper. In practice, the strands of research look rather different.
The core challenge of mechanism design is that agents have private
information or preferences. As we take the perspective of an agent designer, private information is only
relevant for us to the extent that some types of agents or objectives may be
harder to implement than others.
Instead, our core challenge comes from causal relationships in agent
environments, a consideration of little interest to most of mechanism design.

 \section{Discussion and Conclusion}\label{sec:conclusion}
We have proved sound and complete graphical criteria for two existing concepts (VoI and VoC) and
two new concepts: response incentive and instrumental control incentive.
The results have all focused on the (causal) structure of the interaction between
agent and environment.
This is both a strength and a weakness.
On the one hand, it means that formal conclusions can be made about a system's
incentives, even when details about the quantitative relationship
between variables is unknown.
On the other hand, it also means that 
these results will not help with subtler comparisons, such as the relative
strength of different incentives.
It also means that the causal relationships between variables must be
known.
This challenge is common to causal models in general. 
In the context of incentive design, it is partially alleviated by the fact 
that causal relationships often follow directly from the design choices for 
an agent and its objective.
Finally, causal diagrams struggle to express dynamically changing causal
relationships.

While important to be aware of, these limitations do not prevent causal
influence diagrams from providing a clear, useful, and unified perspective on
agent incentives. It has seen applications ranging from
value learning \citep{Armstrong2020pitfalls,Holtman2020},
interruptibility \citep{langlois2021rl}, 
conservatism \citep{cohen2020unambitious},
modeling of agent frameworks \citep{Everitt2019modeling},
and reward tampering \citep{everitt2019tampering}.
Through such applications, we hope that the incentive analysis described in this paper
will ultimately contribute to more fair and safe AI systems.

 \ifblind\else
\section{Acknowledgements}
Thanks to
James Fox,
Lewis Hammond,
Michael Cohen,
Ramana Kumar,
Chris van Merwijk,
Carolyn Ashurst,
Michiel Bakker,
Silvia Chiappa, and
Koen Holtman
for their invaluable feedback.

This work was supported by the Leverhulme Centre for the Future of Intelligence, Leverhulme Trust, under Grant RC-2015-067.

We acknowledge the support of the Natural Sciences and Engineering Research
Council of Canada (NSERC), [funding reference number CGSD3-534795-2019].

Cette recherche a été financée par le Conseil de recherches en sciences
naturelles et en génie du Canada (CRSNG), [numéro de référence CGSD3-534795-2019].
\fi

\bibliography{tex/main}
\newcommand*{\cellcite}[1]{\citeauthor{#1}~\citeyear{#1}}
\rowcolors{2}{gray!20}{white}
\begin{table*}[!th]
    \renewcommand\thetable{S1}
    \centering

\begin{tabular}{p{0.8cm} L{2.5cm} L{3.5cm} L{3.5cm} L{5.3cm}}
& Definition & Criterion & Soundness & Completeness \\
\midrule
VoI &
\cellcite{Howard1966};
\cellcite{Matheson1990} &
\cellcite{fagiuoli1998note};
\cellcite{Lauritzen2001};
\cellcite{Shachter2016} &
\cellcite{fagiuoli1998note};
\cellcite{Lauritzen2001};
\cellcite{Shachter2016} &
First correct proof to our knowledge (see \nameref{sec:related-work})
\\
VoC &
\cellcite{shachter1986evaluating};
\cellcite{Matheson1990};
\cellcite{Shachter2010} &
Incomplete version by \citet{shachter1986evaluating}
\linebreak (see \nameref{sec:related-work}) &
New; proved using do-calculus and VoI &
New; proved constructively
(cf. ``relevant utility nodes'' \citet{nielsen1999welldefined}) \\
RI & New & New & New; proved using do-calculus and VoI & New; proved constructively \\
ICI & New & New & New; proved using do-calculus & New; proved constructively \\
\end{tabular}
\caption{
  Comparison with related work.
  The concepts of positive value of information (VoI), and positive value of
  control (VoC) are well-known. 
    For VoI, a new, corrected, proof is provided.
    For VoC, the present work offers a new criterion, proving it sound and complete.
  For response incentive (RI) and instrumental control incentive (ICI), 
  the criterion and all proofs are new.
}\label{table:prior-work}
\end{table*}
 
\newpage{}
\setcounter{secnumdepth}{2} \appendix

\section{Causality Examples}
\label{app:causality-examples}

\begin{figure}
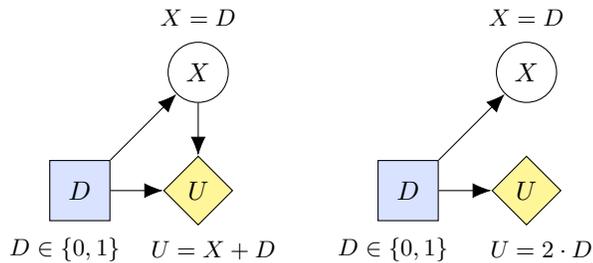

  \centering
  \begin{subfigure}{0.48\linewidth}
    \begin{influence-diagram}
      \node (D) [decision] {$D$};
      \node (U) [right = of D] [utility] {$U$};
      \node (X) [above = of U] {$X$};

      \edge {D} {X};
      \edge {D, X} {U};

    \begin{scope}[
      node distance = 1mm,
      every node/.style = {rectangle, draw=none}]
      \small
        \node [above = of X] {$X=D$};
        \node [below = of U, xshift=2mm] {$U=X+D$};
        \node [below = of D, xshift=-2mm] {$D\in\{0,1\}$};
      \end{scope}
    \end{influence-diagram}
    \caption{A causal influence diagram reflecting the causal structure of the
      environment}
    \label{fig:causality-example-a}
  \end{subfigure}
  \hspace{1mm}
    \begin{subfigure}{0.48\linewidth}
    \begin{influence-diagram}
      \node (D) [decision] {$D$};
      \node (U) [right = of D] [utility] {$U$};
      \node (X) [above = of U] {$X$};

      \edge {D} {X};
      \edge {D} {U};
    \begin{scope}[
      node distance = 1mm,
      every node/.style = {rectangle, draw=none}]
      \small
        \node [above = of X] {$X=D$};
        \node [below = of U, xshift=2mm] {$U=2\cdot D$};
        \node [below = of D, xshift=-2mm] {$D\in\{0,1\}$};
      \end{scope}
    \end{influence-diagram}
    \caption{Influence diagram that is causal in the sense of
      \citet{Heckerman1994,Heckerman1995}}
    \label{fig:causality-example-b}
  \end{subfigure}
  \caption{Two different influence diagram representations of the same
    situation, with different VoC and ICI.}
  \label{fig:causality-example}
\end{figure}

\begin{figure}
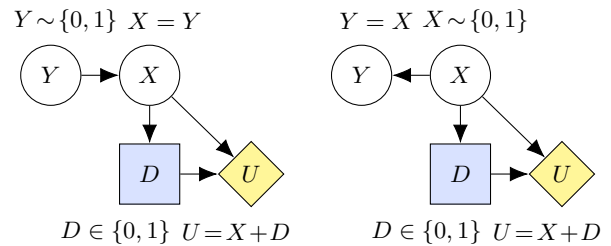

  \centering
  \begin{subfigure}{0.475\linewidth}
    \begin{influence-diagram}
      \small
      \node (D) [decision] {$D$};
      \node (X) [above = 5mm of D] {$X$};
      \node (Y) [left = 5mm of X] {$Y$};
      \node (U) [right = 5mm of D] [utility] {$U$};

      \edge {Y} {X};
      \edge {X} {D};
      \edge {D, X} {U};

      \begin{scope}[
          node distance = 1mm,
          every node/.style = {rectangle, draw=none}]
        \small
        \node [above = of Y, xshift=2mm, yshift=-0.6mm] {$Y\!\sim\! \{0, 1\}$};
        \node [above = of X, xshift=2mm] {$X=Y$};
        \node [below = 1mm of U, xshift=-2mm] {$U\!=\!X\!+\!D$};
        \node [below = 1mm of D, xshift=-4.5mm] {$D\in\{0,1\}$};
      \end{scope}
    \end{influence-diagram}
    \caption{A causal influence diagram reflecting the causal structure of the
    environment}\label{fig:causality-example-ri-a}
  \end{subfigure}
  \hspace{1mm}
  \begin{subfigure}{0.475\linewidth}
    \begin{influence-diagram}
      \small
      \node (D) [decision] {$D$};
      \node (X) [above = 5mm of D] {$X$};
      \node (Y) [left = 5mm of X] {$Y$};
      \node (U) [right = 5mm of D] [utility] {$U$};

      \edge {X} {Y};
      \edge {X} {D};
      \edge {D, X} {U};

      \begin{scope}[
          node distance = 1mm,
          every node/.style = {rectangle, draw=none}]
        \small
        \node [above = of X, xshift=2mm, yshift=-0.6mm] {$X\!\sim\! \{0, 1\}$};
        \node [above = of Y, xshift=2mm] {$Y=X$};
        \node [below = of U, xshift=-2mm] {$U\!=\!X\!+\!D$};
        \node [below = of D, xshift=-4.5mm] {$D\in\{0,1\}$};
      \end{scope}
    \end{influence-diagram}
    \caption{Influence diagram that is causal in the sense of
    \citet{Heckerman1994,Heckerman1995}}\label{fig:causality-example-ri-b}
  \end{subfigure}
  \caption{Two different influence diagram representations of the same
    situation, with different RI and VoC.
    In \cref{fig:causality-example-ri-a},
    $Y$ is sampled from some arbitrary distribution on $\{0, 1\}$, for example a
    Bernoulli distribution with $p=0.5$.
    In \cref{fig:causality-example-ri-b}, $X$ is sampled in the same way.
}
  \label{fig:causality-example-ri}
\end{figure}

Causal influence diagrams that reflect the full causal structure of
the environment are needed to correctly capture response incentives,
value of control and instrumental control incentives.
We begin with showing this for instrumental control incentives and value of control,
leaving response incentive to the end of this section.
Consider the two influence diagrams in \cref{fig:causality-example}.
If we assume that $X$ really affects $U$, only the diagram in
\cref{fig:causality-example-a} correctly represents this causal structure, whereas
\cref{fig:causality-example-b} lacks the edge $X\to U$.
According to \cref{def:instrumental-goal,def:control-incentive-sa}, $X$ has
positive value of control and an instrumental control incentive.
Only \cref{fig:causality-example-a} gets this right.

The influence diagram literature has discussed weaker notions of causality,
under which \cref{fig:causality-example-b} is considered a valid alternative
representation of the situation described by \cref{fig:causality-example-a}.
For example, if we only consider their joint distributions conditional on
various policies, then
\cref{fig:causality-example-a,fig:causality-example-b}
are identical.
Both diagrams are also in the canonical form of
\citet{Heckerman1995}, as every variable responsive to the
decision is a descendant of the decision.
For the same reason,
both diagrams are also causal influence diagrams in the terminology of
\citet{Heckerman1994} and \citet{Shachter2010}.
Since only \cref{fig:causality-example-a} gets the incentives right, we see that
the stronger notion of causal influence diagram introduced in this paper is
necessary to correctly model instrumental control incentives and value of control.

To show that response incentives also rely on fully causal influence diagrams,
consider the diagrams in \cref{fig:causality-example-ri}.
Again, we assume that \cref{fig:causality-example-ri-a}
accurately depicts the environment, while
\cref{fig:causality-example-ri-b} has the edge $Y\to X$ reversed.
Again, both diagrams have identical joint distributions
given any policy.
Both diagrams are also causal in the weaker sense of
\citet{Heckerman1994} and \citet{Shachter2010}.
Yet only the fully causal influence diagram in \cref{fig:causality-example-ri-a}
exhibits that $Y$ can have a response incentive or positive value of control.

\pagebreak
\section{Proof Preliminaries}
\label{sec:appendix_prelims}

Our proofs will rely on the following fundamental results about causal models from
\citep{galles1997axioms} and \citep{Pearl2009}.~\looseness=-1

\begin{definition}[Causal Irrelevance]\label{def:causal-irrelevance}
$\sX$ is \emph{causally irrelevant} to $\sY$, given $\sZ$, written
$(\sX \irrelevant \sY | \sZ)$ if, for every set $\sW$ disjoint of
$\sX \cup \sY \cup \sZ$, we have
\begin{align*}
  \forall \eps, \sz, \sx, \sx', \sw \qquad
  \sY_{\sx\sz\sw}(\eps) &= \sY_{\sx'\sz\sw}(\eps)
\end{align*}
\end{definition}
 \begin{lemma}\label{theorem:path-causal-irrelevance}
For every SCM $\causalmodel$ compatible with a DAG $\causalgraph$,
\begin{align*}
  {(\sX \nopathto \sY | \sZ)}_\causalgraph \Rightarrow
  (\sX \irrelevant \sY | \sZ)
\end{align*}
\end{lemma}
 \begin{proof*}
By induction over variables, as in \cite[Lemma~12]{galles1997axioms}.
\end{proof*}

\begin{lemma}[{\citealp[Thm. 3.4.1, Rule 1]{Pearl2009}}]\label{theorem:do-calc-insertion-of-obs}
For any disjoint subsets of variables $\sW, \sX, \sY, \sZ$ in the DAG
$\causalgraph$,
$\EE(\sY_{\sx} | \sz, \sw) = \EE(\sY_{\sx} | \sw)$
if ${\sY \dsepag{\causalgraph'} \sZ | (\sX, \sW)}$ in the graph $\causalgraph'$
formed by deleting all incoming edges to $\sX$.
\end{lemma}
 \begin{lemma}[{\citealp[Thm. 1.2.4]{Pearl2009}}]\label{theorem:d-separation-independence}
For any three disjoint subsets of nodes $(\sX, \sY, \sZ)$ in a DAG
$\causalgraph$, $(\sX \dsep_\causalgraph \sY | \sZ)$ if and only if
${(\sX \indep \sY | \sZ)}_P$ for every probability function $P$ compatible with $\causalgraph$.
\end{lemma}
 \begin{lemma}[{\citealp[Sigma Calculus Rule 3]{correacalculus}}]\label{le:sigma-calculus-intervention}
For any disjoint subsets of nodes $(\sX, \sY) \subseteq \sV$ and $\sZ \subseteq \sV$ in a DAG $\causalgraph$
$\Pr(\sX | \sZ;g^Y) = \Pr(\sX|\sZ;g'^Y)$
if $\sX \dsep \sY | \sZ$ in $\causalgraph_{\overline{\sY(\sZ)}}$
where $\sY(\sZ) \subseteq \sY$ is the set of elements in $\sY$ that are not ancestors of $\sZ$ in $\causalgraph$
and $\causalgraph_{\overline{\sW}}$ denotes $\causalgraph$ but with edges incoming to variables in $\sW$ removed.\end{lemma}

\section{Proofs}\label{sec:proofs}

\subsection{Value of Information Criterion}\label{sec:appendix-voi}
\newcommand*{\Padg}{\Pad_\causalgraph}
\newcommand*{\padg}{\pad_\causalgraph}

\newcommand*{\Padr}{\ensuremath{\Pad_\text{min}}}
\newcommand*{\padr}{\ensuremath{\pad_\text{min}}}
\newcommand*{\Padnr}{\ensuremath{\Pad_-}}
\newcommand*{\padnr}{\ensuremath{\pad_-}}
\newcommand*{\tpadnr}{\ensuremath{\varidx{\tilde{\pa}}{\decisionvar}_-}}

\newcommand*{\utildvars}{\utilvarsd}
\newcommand*{\utilndvars}{\utilvarsv{\setminus\decisionvar}}
\newcommand*{\tudvar}{\varidx{\totutilvar}{\decisionvar}}
\newcommand*{\tundvar}{\varidx{\totutilvar}{\setminus\decisionvar}}

First, we introduce the notion of a $\rcid$-respecting optimal policy. Our proof of its
optimality is similar to Theorem 3 from \citep{Lauritzen2001}.
It builds on the following intersection property of d-separation.
\begin{lemma}[d-separation \emph{intersection} property]
  \label{le:d-sep-intersection-property}
  For all disjoint sets of variables $\sW$, $\sX$, $\sY$, and $\sZ$,
  \begin{equation*}
    (\sW \dsepg \sX | \sY, \sZ) \land
    (\sW \dsepg \sY | \sX, \sZ) \Rightarrow
    (\sW \dsepg (\sX \cup \sY) | \sZ)
  \end{equation*}
\end{lemma}
 \begin{proof}
Suppose that the RHS is false, so there is a path from $\sW$ to $\sX \cup
\sY$ conditional on $\sZ$.
This path must have a sub-path that passes from $\sW$ to $X \in \sX$ without
passing through $\sY$ or to $Y \in \sY$ without passing through $\sX$ (it must
traverse one set first).
But this implies that $\sW$ is d-connected to $\sX$ given $\sY,\sZ$ or to $\sY$
given $\sX,\sZ$, meaning the LHS is false.
So if the LHS is true, then the RHS must be true.
\end{proof}

\begin{lemma}[$\rcid$-respecting optimal policy] \label{le:reduced-optimal-policy}
    Every single-decision SCIM $\scim = \scimdef$ has an optimal policy
    $\tilde\pi$ that depends only on requisite observations.
    In other words, $\tilde\pi$ is also a policy for the minimal model
    $\rscim = {\left\langle \rcid, \exovars, \structfns, \exoprob \right\rangle}$.
    We call $\tilde\pi$ a \emph{$\rcid$-respecting optimal policy}.
\end{lemma}

\begin{proof}
First partition $\Padg$ into the requisite parents
$\Padr = \{W \in \Pad:W \not \dsepg \utildvars \mid \{\decisionvar\} \cup \Pad \setminus \{W\}\}$, and non-requisite parents $\Padnr = \Padg \setminus \Padr$.

Let $\pi^*$ be an optimal policy in $\scim$.
To construct a $\rcid$-respecting version $\tilde\pi$,
select any value $\tpadnr \in \dom(\Padnr)$ for which $\Prs{\pi^*}(\Padnr = \tpadnr) > 0$.
For all $\padr\in\dom(\Padr)$ and $\exovalv{\decisionvar}\in\dom(\exovarv{\decisionvar})$, let
\begin{align*}
  \tilde\pi (\padr, \padnr, \exovalv{\decisionvar})
  \coloneqq \pi^*(\padr ,\tpadnr, \exovalv{\decisionvar}).
\end{align*}

The policy $\tilde\pi$ is permitted in $\rscim$ because it does not vary with $\Padnr$.

Now let us prove that $\tilde \pi$ that is optimal in $\scim$.
Partition $\utilvar$ into $\utildvars = \utilvars \cap \Descv{\decisionvar}$
and $\utilndvars = \utilvars \setminus \Descv{\decisionvar}$.
$\decisionvar$ is causally irrelevant for every $\utilvar \in \utilndvars$
so every policy $\pi$ (in particular, $\tilde\pi$) is optimal with respect to
$\tundvar \coloneqq \sum_{\utilvar \in \utilndvars}{\utilvar}$.

We now consider $\utildvars$.
By definition, $W \dsepg \utildvars \mid \{\decisionvar\} \cup \Pad \setminus \{W\}$
for every ${W \in \Padnr}$.
By inductively applying the intersection property of d-separation (\cref{le:d-sep-intersection-property}) over elements
of $\Padnr$ we obtain
\begin{equation}\label{eq:padx-indep-tudvar}
    \Padnr \dsep \utildvars \mid \{\decisionvar\} \cup \Padr.
\end{equation}

Next, we establish that
$\EEs{\tilde\pi}[\tudvar] = \EEs{\pi^*}[\tudvar]$ by showing that
${\EEs{\tilde\pi}[\tudvar \mid \pad]} =
{\EEs{\pi^*}[\tudvar \mid \pad]}$ for every
$\pad \in \dom(\Pad)$
with $\Pr(\pad) > 0$.
First, the expected utility of $\tilde\pi$ given any $(\padr, \padnr)$
with $\Pr({\Padr = \padr}, {\padnr = \padnr}) > 0$ is equal
to the expected utility of $\pi^*$ on input $(\padr, \tpadnr)$:
\begin{align*}
  \mathrlap{\EEs{\tilde\pi}[\tudvar \mid \padr, \padnr]}
  \hspace{0.5cm} &
  \\
  &= \sum_{u, \decisionval}{\begin{aligned}[t]\Big(
    &u
    \Pr(\tudvar = u \mid \decisionval, \padr, \padnr) \\
    &\cdot \Prs{\tilde\pi}(\decisionvar = d \mid \padr, \padnr)
  \Big)\end{aligned}}
  \\
  &
  = \sum_{u, \decisionval}{\begin{aligned}[t]\Big(
    &u
    \Pr(\tudvar = u \mid \decisionval, \padr, \tpadnr) \\
    &\cdot \Prs{\pi^*}(\decisionvar = \decisionval \mid \padr, \tpadnr)
  \Big)\end{aligned}}
  \\
  &
    = \EEs{\pi^*}[\tudvar \mid \padr, \tpadnr]
    \intertext{where the middle equality follows from \cref{eq:padx-indep-tudvar} and the
  definition of $\tilde\pi$.
  Second, the expected utility of $\pi^*$ given input $\tpadnr$ is the same as
    its expected utility on any input $\padnr$:
}
  &
  = \max_{\decisionval}{
    \EEs{\pi^*}[\tudvar_\decisionval \mid \padr, \tpadnr]
  }
  \\
  &
  = \max_{\decisionval}{
    \EEs{\pi^*}[\tudvar_\decisionval \mid \padr, \padnr]
  }
  \\
  &
  = \EEs{\pi^*}[\tudvar \mid \padr, \padnr ]
\end{align*}
where the first equality follows from the optimality of $\pi^*$ and
the second from \cref{theorem:do-calc-insertion-of-obs}.
The expression $\EEs{\pi^*}[\tudvar_\decisionval \mid\cdots]$ means that we
first assign the policy $\pi^*$ then intervene to set
$\decisionvar = \decisionval$, which renders $\pi^*$ effectively irrelevant but
formally necessary for creating an SCM.\@
This result shows that $\tilde\pi$ is optimal for $\tudvar$ and has
${\EEs{\tilde\pi}[\tudvar]} = {\EEs{\pi^*}[\tudvar]}$.
Since $\tilde\pi$ is optimal for both $\totutilvard$ and $\tundvar$, $\tilde\pi$ is optimal in $\scim$.
\end{proof}

We now prove \cref{th:observation-sa} by establishing the soundness
and completeness of the value of information criterion.

\begin{lemma}[VoI criterion soundness]\label{le:voi-soundness}
If, in the single-decision CID $\cid$, $X \in \evars \setminus \Descv{D}$ has
\begin{equation*}
    X \perp \utilvarsd \bmid \left(\Pad \cup \{D\} \setminus \{X\}\right)
\end{equation*}
where $\utilvarsd \coloneqq \utilvars\cap\Descv{D}$,
then $\incentivevar$ does not have positive value of information in any SCIM
$\scim$ compatible with $\cid$.
\end{lemma}
 
The result is already known from \citep{Lauritzen2001,fagiuoli1998note}, but we prove it here to make the paper more self-contained.
\begin{proof}
    Let $\scim = \langle \causalgraph,\exovars,\structfns,\exoprob \rangle$ be any SCIM compatible with $\causalgraph$.
    Let $\causalgraph_{X \to D}$ and $\causalgraph_{X \not \to D}$ be versions of $\causalgraph$ modified by adding and
    removing $X \to D$ respectively.
    Let $\reducedgraph_{X \to D}$ be the minimal reduction of $\causalgraph_{X \to D}$.
    Let $\scim_{X \not \to D}:=\langle \causalgraph_{X \not \to D},\exovars,\structfns,\exoprob \rangle$ and
    $\rscim_{X \to D}:=\langle \reducedgraph_{X \to D},\exovars,\structfns,\exoprob \rangle$ be SCIMs with the same domains and structural functions.

    By \cref{le:reduced-optimal-policy}, there is a $\reducedgraph$-respecting policy $\tilde \pi$
    admissible in $\rscim_{X \to D}$ and optimal in $\scim_{X \to D}$.
    We prove that $\reducedgraph_{X \to D}$ is a subgraph of $\causalgraph_{X \not \to D}$,
    meaning that $\tilde \pi$ is also admissible in $\scim_{X \not \to D}$.
    By assumption, $\causalgraph$ has $X \perp \utildvars \bmid \left(\Pad \cup \{D\} \setminus \{X\}\right)$.
    Adding $X \to D$ to $\causalgraph$ cannot cause $X$ to be d-connected to $\utildvars$ given $\Pad \cup \{D\}$,
    because any new path along $X \to D$ is blocked by $D$ and $\Pad \setminus \{X\}$.
    So $\reducedgraph_{X \to D}$ is a version of $\causalgraph$ with $X \to D$ (and possibly other nodes) removed.
    This makes it a subgraph of $\causalgraph_{X \not \to D}$, implying that $\tilde \pi$ is admissible in $\scim_{X \not \to D}$.

    Since $\tilde \pi$ is admissible in $\scim_{X \not \to D}$ and optimal in $\scim_{X \to D}$,
    $\attutil(\scim_{X \not \to D}) \not < \attutil(\scim_{X \to D})$.
\end{proof}

\begin{lemma}[VoI criterion completeness]\label{le:voi-completeness}
If in the single-decision CID $\cid$, $\incentivevar\in\evars \setminus\Descv{\decisionvar}$ is
d-connected to a utility node that is a descendant of $\decisionvar$ conditional on the
decision and other parents:
\begin{equation}
    \incentivevar \not \perp \utilvarsd \bmid \left(\Pad \cup \{\decisionvar\} \setminus \{\incentivevar\}\right)
\end{equation}
where $\utilvarsd \coloneqq \utilvars\cap\Descv{\decisionvar}$
then $\incentivevar$ has VoI in at least one SCIM $\scim$
compatible with $\cid$.
\end{lemma}
 This follows from the response incentive completeness
\cref{theorem:ri-graph-criterion-completeness} in
\cref{sec:appendix-ri}, so we defer the proof to that section.
 
\subsection{Response Incentive Criterion}\label{sec:appendix-ri}
The \nameref{sec:response} section contains a proof of the soundness of the
response incentive criterion.
We now prove its completeness in order to finish the proof of
\cref{theorem:ri-graph-criterion}.
\Cref{fig:ri-completeness-2} illustrates the model constructed in the proof.

\begin{lemma}[Response Incentive Criterion Completeness]\label{theorem:ri-graph-criterion-completeness}
If $\incentivevar \pathto \decisionvar$ in the minimal reduction $\reducedgraph$
of a single-decision CID $\cid$ then there is a response incentive on
$\incentivevar$ in at least one SCIM $\scim$ compatible with $\cid$.
\end{lemma}
 \newcommand*{\pwu}{{\overline{\riparvar\utilvar}}}
\newcommand*{\pdu}{{\overrightarrow{\decisionvar\utilvar}}}
\newcommand*{\pxd}{\overrightarrow{\incentivevar\decisionvar}}
\newcommand*{\wusrc}{S}
\newcommand*{\wucol}{C}
\newcommand*{\wuobs}{O}
\newcommand*{\pco}[1]{\ensuremath{\overrightarrow{\wucol^{#1}\wuobs^{#1}}}}
\newcommand*{\riinter}{\ensuremath{\incentivevar = 0}}

\begin{proof}
\begin{figure*}
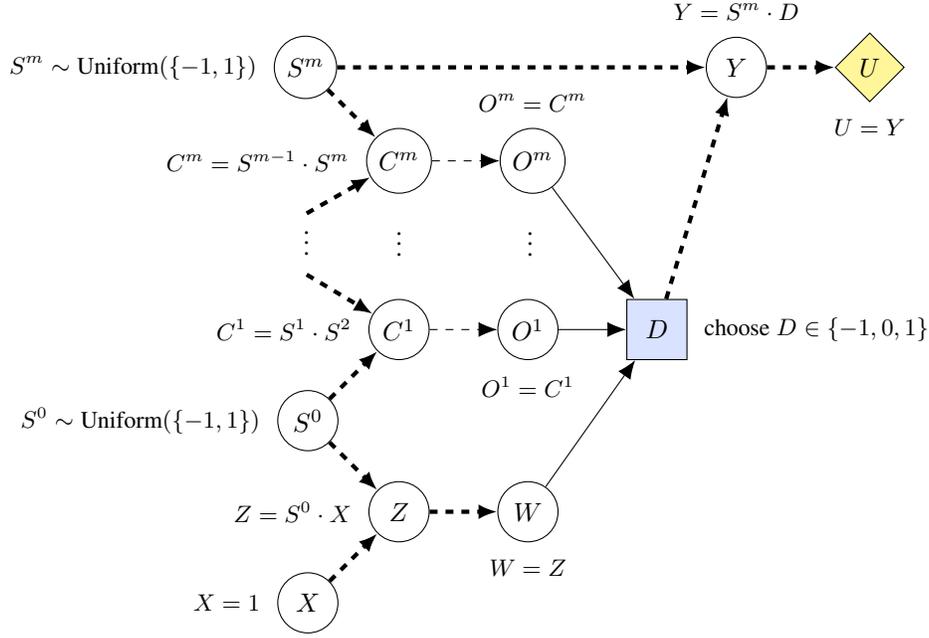

  \centering
    \begin{influence-diagram}[
      every node/.append style = {circle},
      node distance=9mm
  ]
    \node[] (Sl) {$\wusrc^m$};
    \node[below right = of Sl] (Olp) {$\wucol^{m}$};
    \node[right = of Olp] (Ol) {$\wuobs^{m}$};
    \node[below = of Olp, yshift=-5mm] (\wuobs1p) {$\wucol^{1}$};
    \node[right = of O1p] (O1) {$\wuobs^{1}$};
    \node[below left = of O1p] (S0) {$\wusrc^0$};
    \node[below right = of S0, draw=none] (Oh)  {};
    \node[right = of Oh] (O) {$\riparvar$};
    \node[left = of O] (Z) {$Z$};
    \node[below left = of Z] (X) {$\incentivevar$};

    \node[draw=none] (sd) at ($(Sl)!0.5!(S0)$)  {$\rvdots$};
    \node[draw=none] at ($(Ol)!0.5!(O1)$)  {$\rvdots$};
    \node[draw=none] at ($(Olp)!0.5!(O1p)$)  {$\rvdots$};

    \node[decision, right = of O1] (A)  {$\decisionvar$};
    \node[right = of Sl, xshift=4cm] (Up)  {$Y$};
    \node[right = of Up, utility] (U)  {$\utilvar$};

    \path[dashed, ultra thick]
    (Sl) edge[->] (Up)
    (Sl) edge[->] (Olp)
    (S0) edge[->] (O1p)
    (S0) edge[->] (Z)
    (Z) edge[->] (O)
    (X) edge[->] (Z)
    (A) edge[->] (Up)
    (Up) edge[->] (U)
    (sd.south) edge[->] (O1p)
    (sd.north) edge[->] (Olp)
    ;

    \path[dashed]
    (Olp) edge[->] (Ol)
    (O1p) edge[->] (O1)
    ;

    \path
    (Ol) edge[->, information] (A)
    (O1) edge[->, information] (A)
    (O) edge[->, information] (A)
    ;

    \begin{scope}[
      node distance = 1mm,
      every node/.style = {rectangle, draw=none},
      ]
    \node[above = of Up] {\small $Y=\wusrc^m \cdot D$};
    \node[below = of O1]  {\small $\wuobs^1=\wucol^1$};
    \node[above = of Ol] {\small $\wuobs^m=\wucol^m$};
    \node[below = of U] {\small $\utilvar=Y$};
    \node[below = of O] {\small $\riparvar = Z$};
    \node[left = of Z] {\small $Z = \wusrc^0 \cdot X$};
    \node[left = of X] {\small $\incentivevar = 1$};
    \node[left = of Olp] {\small $\wucol^m = \wusrc^{m-1}\cdot \wusrc^{m}$};
    \node[left = of Sl] {\small $\wusrc^m\sim \textrm{Uniform}(\{-1,1\})$};
    \node[left = of S0] {\small $\wusrc^0\sim \textrm{Uniform}(\{-1,1\})$};
    \node[left = of O1p] {\small $\wucol^1 = \wusrc^1\cdot \wusrc^2$};
    \node[right = of A] {\small choose $\decisionvar \in \{-1, 0, 1\}$};
    \end{scope}

  \end{influence-diagram}
  \caption{
    Outline of the variables involved in the response incentive construction.
    Every graph that satisfies the response incentive graphical criterion
    contains this structure (allowing all dashed paths except
    those to $\wucol^i$ or $Y$ to have length zero).
    An optimal policy for the given model is
    $D = \riparvar \cdot \prod_i \wuobs^i = \wusrc^m$,
    yielding utility $U = Y = \incentivevar(\wusrc^m)^2 = 1$,
    and all optimal policies must depend on the value of $\riparvar$.
  }\label{fig:ri-completeness-2}
\end{figure*}
Starting from the assumption that $\incentivevar \pathto \decisionvar$ in
$\rcid$, we explicitly construct a compatible model for $\cid$
for which the decision of every optimal policy causally depends on the value of
$\incentivevar$.
Let $\pxd$ be a directed path from $\incentivevar$ to $\decisionvar$
that only contains a single requisite observation that we label
$\riparvar$ (if $X$ is itself a requisite observation, then $W$ and $X$ is the
same node).
Since $\riparvar$ is a requisite observation for $\decisionvar$, there exists
some utility node $\utilvar$ descending from $\decisionvar$ that is d-connected
to $\riparvar$ in $\cid$ when conditioning on $\reqobscond$.
Let $\pdu$ be a directed path from $\decisionvar$ to $\utilvar$ and
let $\pwu$ be a path between $\riparvar$ and $\utilvar$ that is active
when conditioning on $\reqobscond$.
By the definition of d-connecting paths,
$\pwu$ has the following structure ($m \ge 0$):
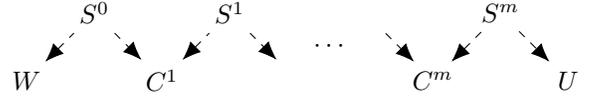
\begin{center}
\begin{tikzpicture}[scale=0.9]
\node at (0, 0) (W) {$\riparvar$};
  \node at (1, 1) (S0) {$\wusrc^0$};
  \draw (W) edge[<-, dashed] (S0);
  \node at (2, 0) (O1) {$\wucol^1$};
  \draw (S0) edge[->, dashed] (O1);
  \node at (3, 1) (S1) {$\wusrc^1$};
  \draw (O1) edge[<-, dashed] (S1);
  \node[minimum height = 1.5em, minimum width = 2em, draw=none] at (4, 0)
    (O2) {};
  \draw (S1) edge[->, dashed] (O2);
  \node at (4.5, 0.5) {$\cdots$};
  \node[minimum height = 1.5em, minimum width = 2em, draw=none] at (5, 1)
    (Sm1) {};
  \node at (6, 0) (Om) {$\wucol^m$};
  \draw (Sm1) edge[->, dashed] (Om);
  \node at (7, 1) (Sm) {$\wusrc^m$};
  \draw (Om) edge[<-, dashed] (Sm);
  \node at (8, 0) (U) {$\utilvar$};
  \draw (Sm) edge[->, dashed] (U);
\end{tikzpicture}
\end{center}
consisting of directed sub-paths leaving source nodes $\wusrc^i$ and entering
collider nodes $\wucol^i$, where there is a directed path from each collider to
$\reqobscond$ and no non-collider node is in $\reqobscond$.
It may be the case that $\riparvar$ and $\wusrc^0$ are the same node.
For each $i \in \{1, \ldots, m\}$, let $\pco{i}$ be a directed path from
$\wucol^i$ to some $\wuobs^i \in \Pad$ such that no other node along $\pco{i}$
is in $\Pad$.

We make the following assumptions without loss of generality:
\begin{itemize}
  \item $\pwu$ first intersects $\pdu$ at some variable $Y$ (possibly $Y$
    is $\utilvar$) and thereafter both $\pwu$ and $\pdu$ follow the same directed
    path from $Y$ to $U$ (otherwise, let $Y$ be the first intersection point and
    replace the $Y \pathto U$ sub-path of $\pwu$ with the $Y \pathto U$ sub-path
    of $\pdu$).
  \item The $\wusrc^0 \pathto W$ sub-path of reversed $\pwu$ first intersects $\pxd$
    at some node $Z$ and thereafter both follow the same directed path from $Z$
    to $\riparvar$ (same argument as for $Y$).
  \item The paths $\pco{i}$ are mutually non-intersecting
    (if there is an intersection between $\pco{i}$ and $\pco{j}$ with $j \ne i$
    then replace the part of $\pwu$ between $\wucol^i$ and $\wucol^j$ with the
    path through the intersection point, which becomes the new collider; this
    can only happen finitely many times as it reduces the number of collider
    nodes).
\end{itemize}
The resulting structure is shown in \cref{fig:ri-completeness-2}.

We now formally define the model represented in the figure.
The domains of all endogenous variables are set to $\{-1, 0, 1\}$.
All exogenous variables are given independent discrete uniform
distributions over $\{-1, 1\}$.
Unless otherwise specified, we set $B = A$ for each edge $A \to B$ within the
directed paths shown in \cref{fig:ri-completeness-2},
i.e. $\fv{B}(\pav{B}, \exovalv{B}) = a$.
Nodes at the heads of directed paths can therefore be defined in terms of nodes
at the tails.
We begin by describing functions for the ``default'' case depicted by
\cref{fig:ri-completeness-2}, and discuss adaptations for various special cases
below.
\begin{itemize}
\item $\wusrc^i = \exovarv{\wusrc^i}$, giving
  $\wusrc^i$ a uniform distribution over $-1$ and $1$.
\item $U = Y$, and
\item $Y = \wusrc^m \cdot D$, so $D$ must match $S^m$ to optimize utility.
\item $\wucol^i = \wusrc^{i - 1} \cdot \wusrc^i$, and
\item $\wuobs^i = \wucol^i$, so the collider $\wucol^i$
  reveals (only) whether $\wusrc^{i - 1}$ and $\wusrc^i$ have the same sign or
  not.
\item $X=1$,
\item $Z=X\cdot S^0$, and
\item $W=Z$, so $W$ reflects the value of $S^0$, unless $X$ is intervened upon.
\end{itemize}
All other variables not part of any named path are set to $0$.

Special cases arise when two or more of the labeled nodes in 
\cref{fig:ri-completeness-2} refer to the same variable.
When $W$, $Y$, or $O^i$ is the same node as one of its parents,
then it simply takes the function of this parent (instead of copying its value).
Meanwhile, the $S^i$, $C^i$, and $Y$ nodes must be distinct by construction, so
no special cases treatment is required.
Finally, the functions for $X$, $S^0$ and $Z$ are adapted per the following cases:

\emph{Case 1: $X$, $S^0$, and $Z$ are all the same node}.
Let
$X=Z=S^0=\exovarv{S^0}$, i.e.\ the node takes a uniform distribution over
$\{-1, 1\}$.

\emph{Case 2: $Z$ is the same node as $S^0$, but different from $X$}.
In this case, let $Z=S^0=X\cdot \exovarv{S^0}$.

\emph{Case 3: $X$ is the same node as $Z$, but different from $S^0$}.
In this case, let $X=Z=S^0$.

The final combination of $X$ and $S^0$ being the same, while different from $Z$,
cannot happen by the definition of $Z$.

Regardless of which case applies, an optimal policy is
$\decisionvar = \riparvar \cdot \prod_{i=1}^m{\wuobs^i}$,
which yields a utility of $1$.

Now consider the intervention that sets ${\incentivevar = 0}$,
and consequently $\riparvar_{\riinter} = Z_{\riinter} = 0$.
Without the information in $\riparvar$, $\wusrc^m$ is independent of
${(\Pad)}_{\riinter}$ and hence independent of $\decisionvar_{\riinter}$
regardless of the selected policy.\footnote{
  Note that if $m = 0$ and $\wusrc^0$ is $Z$ then
  ${(\wusrc^m)}_{\riinter} = 0$ but the fact that this is predictable is
  irrelevant because we compare
  $\decisionvar_{\riinter}$ against the pre-intervention variable $\wusrc^m$,
  which remains independent of ${(\Pad)}_{\riinter}$.}
Therefore,
$\EEs{\pi}[\utilvar_{\decisionvar_{\riinter}}]
= \EEs{\pi}[\wusrc^m \cdot \decisionvar_{\riinter}]
= \EEs{\pi}[\wusrc^m] \cdot \EEs{\pi}[\decisionvar_{\riinter}]
= 0$
for every
policy $\pi$.
In particular, for any optimal policy $\pi^*$,
$\EEs{\pi^*}[\utilvar_{\decisionvar_{\riinter}}] \ne
\EEs{\pi^*}[\utilvar] = 1$
so there must be some $\exovals$ such that
$\decisionvar_{\riinter}(\exovals) \ne \decisionvar(\exovals)$.
Therefore, there is a response incentive on $\incentivevar$.
\end{proof}

With this result we can now prove the completeness of the value of information
criterion.
\begin{proof}[Proof of \cref{le:voi-completeness} (VoI criterion completeness)]
If $\incentivevar \not\dsep \utilvarsd \mid
(\Pad \cup \{\decisionvar\} \setminus \{\incentivevar\})$
then $\incentivevar$ is a requisite observation in
$\causalgraph_{\incentivevar \to \decisionvar}$
(where $\causalgraph_{\incentivevar \to \decisionvar}$ is $\causalgraph$
modified to include the edge $\incentivevar \to \decisionvar$ if the edge does
not exist already) and $\incentivevar \to \decisionvar$ is a path in
the minimal reduction $\reducedgraph_{\incentivevar \to \decisionvar}$.
By \cref{theorem:ri-graph-criterion-completeness},
there exists a model $\scim_{\incentivevar \to \decisionvar}$ compatible with
$\causalgraph_{\incentivevar \to \decisionvar}$ that has a response incentive on
$\incentivevar$.
If every optimal policy for $\scim_{\incentivevar \to \decisionvar}$ depends on
$\incentivevar$ then it must be the case that
$\attutil(\scim_{\incentivevar \not\to \decisionvar})
< \attutil(\scim_{\incentivevar \to \decisionvar})$.
\end{proof}

\subsection{Value of Control Criterion}\label{sec:appendix-voc}

The \nameref{sec:control-incentives} section contains a proof of the
soundness of the value of control criterion.
We complete the proof of \cref{th:soft-sa} by showing that the criterion is also
complete.

\begin{lemma}[VoC criterion completeness]\label{le:voc-completeness}
If $\incentivevar \pathto \utilvars$ in the minimal reduction $\rcid$ of
a single-decision CID $\cid$ and $\incentivevar \not\in \{\decisionvar\}$ then
$\incentivevar$ has positive value of control in at least one SCIM $\scim$
compatible with $\cid$.
\end{lemma}
 \begin{proof}
Assume that $\incentivevar \pathto \utilvars$ for
$\incentivevar \notin \{\decisionvar\}$
and fix a particular directed path $\rho$ from $\incentivevar$ to some utility
$\utilvar \in \utilvars$.
We consider two cases depending on whether $\decisionvar$ is in $\rho$
and construct a SCIM for each:

\emph{Case 1: $\rho$ does not contain $\decisionvar$.}
Let the domain of all variables be $\{0, 1\}$.
Set all exogenous variable distributions arbitrarily.
Set $\structfns$ such that $\incentivevar = 0$ with every other variable along
$\rho$ copying the value of $\incentivevar$ forward. All remaining variables
are set to the constant $0$.
With this model, an intervention $\gx$ that sets $\incentivevar$ to
$1$ instead of $0$ increases the total expected utility by $1$, which means
there is an instrumental control incentive for $\incentivevar$.

\emph{Case 2: $\rho$ contains $\decisionvar$.}
This implies that a directed path $\incentivevar \to \decisionvar$ is present in
$\rcid$ so we can construct (a modified version of) the response incentive
construction used in the proof of
\cref{theorem:ri-graph-criterion-completeness}.
We make one change: instead of starting with
$\fv{\incentivevar}(\cdot) = 1$ we start with
$\fv{\incentivevar}(\cdot) = 0$.
As noted in the response incentive completeness proof, this means that
$S_m$ is independent of $\Pad$ so regardless of the policy the optimal
attainable utility is $0$.
If we perform the intervention $\gx(\cdot) = 1$ then the attainable
expected utility is $1$ once again so the intervention $\gx$
strictly increases the optimal expected utility.
\end{proof}

\subsection{Instrumental Control Incentive Criterion}\label{sec:appendix-fci}

The \nameref{sec:fci} section contains a proof of the soundness of the
instrumental control incentive criterion.
We prove its completeness to finish the proof of
\cref{theorem:ig-graph-criterion}.

\begin{lemma}[ICI Criterion Completeness]\label{theorem:fci-graph-criterion-completeness}
If a single-decision CID $\cid$ contains a path of the form
$\decisionvar \pathto \incentivevar \pathto \utilvars$ then
there is an instrumental control incentive on $\incentivevar$ in at least one SCIM $\scim$
compatible with $\cid$.
\end{lemma}
 \begin{proof}\label{le:fci-completeness}
Assume that $\cid$ contains a directed path
$\decisionvar = Z^0 \to Z^1 \to \cdots \to Z^n = \utilvar$ where
$\utilvar \in \utilvars$ and
$Z^i = \incentivevar$ for some $i \in \{0, \ldots, n\}$.
We construct a compatible SCIM for which there is an instrumental control incentive
on $\incentivevar$.
Let all variables along the path $Z^0 \to \ldots \to Z^n$ be
equal to their predecessor, except $Z^0 = \decisionvar$, which has no structure
function. All other variables are set to 0.
In this model, $\utilvar \ceq \decisionvar \in \{0, 1\}$ and all other utility
variables are always $0$ so the only optimal policy is
$\pi^*(\pad) = 1$, which gives
${\EEs{\pi^*}[\totutilvar \mid \Pad=\bm{0}] = 1}$.
Meanwhile, $\utilvar_{\incentivevar_{\decisionval}} = \decisionval$
so for $\decisionval = 0$ we have
${\EEs{\pi^*}[\totutilvar_{\incentivevar_{\decisionval}} \mid
\Pad = \bm{0}] = 0}$.
\end{proof}
 
\subsection{Counterfactual Fairness}\label{sec:appendix-fairness}
\theoremcffair*
\newcommand*{\supps}[1]{\ensuremath{\supp_{#1}}}
\begin{proof}
  We begin by showing that if there exists an optimal policy $\pi$ that is
  counterfactually fair, then there is no response incentive on $A$.
To this end, let
  \begin{align*}
    \supps{\pi}(D\mid \pad) &= \{ d\mid \Prs{\pi}(D = d\mid\pad)>0 \}\\
    \forall a,\  \supps{\pi}(D_a\mid \pad) &= \{ d\mid \Prs{\pi}(D_a = d\mid\pad)>0 \}
  \end{align*}
  be the sets of decisions taken by $\pi$ with positive probability with and
  without an intervention on $A$.
  As a first step, we will show that for any $\exovals\in\dom(\exovars)$ and
  any intervention $a$ on $A$,
  \begin{equation}
    \label{eq:supp}
    \supps{\pi}\big(D\mid \Pad(\exovals)\big)
    = \supps{\pi}\big(D_a\mid \Pad(\exovals)\big).
  \end{equation}
  By way of contradiction, suppose there exists a decision
  \begin{equation}
    \label{eq:set-diff}
    d\in
    \supps{\pi}\big(D\mid \Pad(\exovals)\big)\setminus
    \supps{\pi}\big(D_a \mid \Pad(\exovals)\big).
\end{equation}
  Since $d\in \supps{\pi}\big(D\mid \Pad(\exovals)\big)$,
  we have
  \begin{equation}
    \label{eq:pad-geq}
    \Prs{\pi}\left(D=d\mid \Pad(\exovals), A(\exovals)\right)> 0.
  \end{equation}
  And since $d\not\in\supps{\pi}\big(D_a\mid \Pad(\exovals)\big)$,
  there exists no $\exovals'$ with positive probability
  such that $\Pad(\exovals') = \Pad(\exovals)$,
  $A(\exovals') = A(\exovals)$,
  and $D_a(\exovals') = d$.
  This gives
  \begin{equation}
    \label{eq:pada-eq}
    \Prs{\pi}\left(D_{a}=d\mid \Pad(\exovals), A(\exovals)\right) = 0.
  \end{equation}
  \Cref{eq:pad-geq,eq:pada-eq} violate the counterfactual fairness property,
  \cref{eq:cf},
  which shows that \cref{eq:set-diff} is impossible.
  An analogous argument shows that
  $d\in \supps{\pi}\big(D_a\mid \Pad(\exovals)\big)\setminus
  \supps{\pi}\big(D\mid \Pad(\exovals)\big)$
  also violates the counterfactual fairness property \cref{eq:cf}.
  We have thereby established \cref{eq:supp}.

  Now select an arbitrary ordering of the elements of $\dom(D)$ and define a new
  policy $\pi^*$ such that $\pi^*(\pad)$ is the minimal element of
  $\supps{\pi}(D\mid \pad)$.
  Then $\pi^*$ is optimal because $\pi$ is optimal.
  Further, $\pi^*$ will make the same decision in
  decision contexts $\Pad(\exovals)$ and $\Pad_a(\exovals)$
  because of \cref{eq:supp}.
  In other words, $D_a(\exovals) = D(\exovals)$ in $\scim_{\pi^*}$ for the
  optimal policy $\pi^*$, which means that there is no response incentive on
  $A$.

  Now we prove the reverse direction --- that if there is no response incentive then some optimal $\pi^*$ is counterfactually fair.
  Choose any optimal policy $\pi^*$ where $D_a(\exovals)=D(\exovals)$ for all $\exovals$.
  Since an intervention $a$ cannot change $D$ in any setting, $\Pr(D_{a}=d\mid
  \cdot)=\Pr(D=d\mid \cdot)$ for any condition and any decision $d$, hence $\pi^*$ is counterfactually fair.
\end{proof}

\end{document}